\documentclass{article}

\PassOptionsToPackage{numbers, sort&compress}{natbib}

\usepackage{natbib}

\usepackage[preprint]{neurips_2024}

\usepackage{pgfplots}
\pgfplotsset{compat=1.17} 
\usepackage{caption}
\usepackage{subcaption}
\usepackage{svg}
\usepackage{tabularx}
\usepackage{booktabs} 

\usepackage[utf8]{inputenc} 
\usepackage[T1]{fontenc}    
\usepackage{hyperref}       
\usepackage{url}            
\usepackage{booktabs}       
\usepackage{amsfonts}       
\usepackage{amsthm}
\usepackage{mathtools}
\usepackage{nicefrac}       
\usepackage{microtype}      
\usepackage{xcolor}         
\usepackage{xspace}
\usepackage{xparse}         
\usepackage{cleveref}       
\usepackage{stmaryrd}
\SetSymbolFont{stmry}{bold}{U}{stmry}{m}{n}
\usepackage{graphicx}       
\usepackage{floatrow}
\usepackage{dsfont}
\usepackage{bm}

\usepackage{amsmath}
\usepackage{amssymb}
\usepackage{mathtools}
\usepackage{amsthm}
\usepackage{bm}
\usepackage{nicefrac} 
\usepackage{enumitem}
\usepackage{color}
\usepackage{multirow}
\usepackage{nccmath, amssymb}
\usepackage{bbold}
\usepackage{thmtools,thm-restate}
\usepackage{cases}

\usepackage{etoolbox,siunitx}
\robustify\bfseries
\usepackage{nicefrac} 
\usepackage{enumitem}
\usepackage{color}
\usepackage{multirow}
\usepackage{nccmath, amssymb}
\usepackage{bbold}
\usepackage{thmtools,thm-restate}

\newtheorem{theorem}{Theorem}
\newtheorem{lemma}{Lemma}
\newtheorem{remark}{Remark}

\newtheorem{proposition}{Proposition}

\DeclareMathOperator{\tr}{\ensuremath{\text{\rm tr}}}
\DeclareMathOperator{\Diag}{\ensuremath{\text{\rm diag}}}

\providecommand{\norm}[1]{\lVert#1\rVert}
\providecommand{\SVDr}[1]{[\![#1]\!]_r}

\newcommand{\scalarp}[1]{{\langle #1\rangle}}

\newcommand{\R}{\mathbb R}

\newcommand{\EE}{\ensuremath{\mathbb E}}

\newcommand{\Id}{I}
\newcommand{\Data}{\mathcal{D}_n}
\newcommand{\Koop}{{A_{\im}}}  

\newcommand{\EKRR}{\widehat{G}_\reg} 
 
\newcommand{\EEstim}{\widehat{G}}  
\newcommand{\Estim}{G}  

\newcommand{\ERRR}{\EEstim^{\rm RRR}_{r,\reg}}  

\newcommand{\ECx}{\widehat{C} } 
\newcommand{\ECy}{\widehat{D}} 
\newcommand{\ECxy}{\widehat{T}}  
 
\newcommand{\ECreg}{\widehat{C}_{\reg}}

\newcommand{\X}{\mathcal{X}} 
\newcommand{\Risk}{\mathcal{R}} 
\newcommand{\ERisk}{\widehat{\mathcal{R}}} 
\newcommand{\RKHS}{\mathcal{H}} 
\newcommand{\Lii}{L^2_{\im}(\X)}

\newcommand{\im}{\pi} 
\newcommand{\HS}[1]{{\rm{HS}}\left(#1\right)} 

\newcommand{\reg}{\gamma}

\newcommand{\bcon}{c_\RKHS}

\newcommand{\fH}{\phi}

\newcommand{\Cx}{C}

\newcommand{\Cy}{D}
\newcommand{\Cxy}{T}

\newcommand{\Kreg}[1]{K^{\reg}}

\newcommand{\Stau}{S_{\tau}}
\newcommand{\Stautilde}{\Tilde{S_{\tau}}}
\newcommand{\Corrtau}[1]{V_{\tau}(\mathbf{#1})}
\newcommand{\BECorrtau}[1]{\widetilde{V}_{\tau}(\mathbf{#1})}

\newcommand{\UECorrtau}[1]{\widehat{V}_{\tau}(\mathbf{#1})}
\newcommand{\delt}{\delta(\tau)}
\newcommand{\deltprime}{\delta'(\tau)}
\newtoks\rowvectoks
\newcommand{\rowvec}[2]{%
  \rowvectoks={#2}\count255=#1\relax
  \advance\count255 by -1
  \rowvecnexta}
\newcommand{\rowvecnexta}{%
  \ifnum\count255>0
    \expandafter\rowvecnextb
  \else
    \begin{pmatrix}\the\rowvectoks\end{pmatrix}
  \fi}
\newcommand\rowvecnextb[1]{%
    \rowvectoks=\expandafter{\the\rowvectoks&#1}%
    \advance\count255 by -1
    \rowvecnexta}

\usepackage[disable,textsize=tiny]{todonotes}

\def\mybigtimes{\mathop{\mathchoice{
   \vcenter{\hbox to10bp{\vrule height15bp width0pt \pdfliteral{
   q 1 J .8 w 0 1 m 10 14 l S 0 14 m 10 1 l S Q
}\hss}}}{
   \vcenter{\hbox to10bp{\kern1bp\vrule height10bp width0pt \pdfliteral{
   q 1 J .65 w 0 0 m 8 10 l S 0 10 m 8 0 l S Q
}\hss}}}{\times}{\times}
}}

\title{An Empirical Bernstein Inequality for Dependent Data in Hilbert Spaces and Applications}

\author{ Erfan Mirzaei \\ {Istituto Italiano di Tecnologia} \\ {University of Genoa}\\ {\tt \small erfan.mirzaei@iit.it}  \And
      Andreas Maurer \\
      Istituto Italiano di Tecnologia\\
      \texttt{am@andreas-maurer.eu}
      \And 
Vladimir R. Kostic \\ {Istituto Italiano di Tecnologia} \\ {University of Novi Sad} \\ {\tt \small vladimir.kostic@iit.it}
      \And      
Massimiliano Pontil \\ {Istituto Italiano di Tecnologia} \\ {University College London} \\ {\tt \small massimiliano.pontil@iit.it}
}

\begin{document}

\maketitle

\begin{abstract}
Learning from non-independent and non-identically distributed data poses a persistent challenge in statistical learning. In this study, we introduce data-dependent Bernstein inequalities tailored for vector-valued processes in Hilbert space. Our inequalities apply to both stationary and non-stationary processes and exploit the potential rapid decay of correlations between temporally separated variables to improve estimation. We demonstrate the utility of these bounds by applying them to covariance operator estimation in the Hilbert-Schmidt norm and to operator learning in dynamical systems, achieving novel risk bounds.  Finally, we perform numerical experiments to illustrate the practical implications of these bounds in both contexts.
\end{abstract}

\section{Introduction}\label{sec:intro}
Learning from non-independent and identically distributed (non-i.i.d.) data presents significant challenges in machine learning, both from theoretical and practical perspectives. 
Most real-world data do not follow the neat, predictable patterns of i.i.d. scenarios, creating a demand for statistical learning techniques to handle more complex random processes, thereby broadening the applicability of learning algorithms.

In this paper, we present data-dependent Empirical Bernstein Inequalities (EBIs), which apply to vector-valued random processes in Hilbert space. A driving motivation for this work is recent studies on learning operators associated with stochastic dynamical systems \citep{Kostic2022,kostic2024consistent}. 

Stochastic dynamical systems are essential for modeling complex phenomena across diverse fields, from finance, where they describe asset price fluctuations and stochastic volatility \citep{tankov2003financial}, to neuroscience, where they capture neural variability and synaptic noise \citep{rusakov2020noisy, schug2021presynaptic}, and climate science, where they model turbulent atmospheric and oceanic dynamics \citep{majda2012filtering}. A key example is Langevin dynamics, which describes molecular motion in a thermal environment by incorporating both deterministic forces and random fluctuations, making it fundamental for simulating biomolecular systems, modeling Brownian motion, and analyzing stochastic processes in physics, chemistry, and electrical engineering \citep{coffey2012langevin}.

Many such processes are slowly exploring the state space, and one has to wait a long time before two points along the process can be considered independent. Such phenomena are formalized via the notion of mixing.
Unfortunately, for largely unknown processes, the quantification of mixing is unfeasible, and it may be hard even when dealing with data generated from simulations based on mathematical models that ensure that the dynamical system is mixing.
This motivates the empirical concentration inequalities for dependent random variables in Hilbert spaces since their primary objective is to minimize the inequality’s dependence on mixing coefficients that are often unknown. 

In this respect, EBIs contrast with the more classical combinations of Bernstein inequalities and mixing assumptions. Current estimation bounds for covariance operators exhibit a reduced effective sample size, which roughly requires dividing the length of the trajectory by the mixing time. We show that our EBIs allow us to derive estimation bounds in which these large mixing times mainly impact the fast $O(1/n)$ term in the bound, while the slow $O(1/\sqrt{n})$ term involves only the time-lag correlation/variance of the within-block average of the process, which may be very small even for slowly mixing processes. 
Notably, when the variables observed along a trajectory decorrelate much more rapidly than they achieve approximate independence, Bernstein's inequality, with its smaller variance term, facilitates $O(1/n)$ convergence rates, a significant improvement over standard ones. 

This utility is enhanced by combining the inequality with precise estimates of its variance term, further differentiating our empirical bounds from others and highlighting the novelty of our approach compared to related works. Finally, we highlight that our bound only requires prior knowledge of the mixing coefficient of the process, with all other quantities being data-dependent. 


{\bf Previous Work~} The idea of combining Bernstein's inequality with estimates of the variance term is not new. It was first applied to reinforcement learning and general learning theory (\citep{audibert2007variance}, \citep{maurer2009empirical}, \citep{audibert2009exploration}) and has been extended to improve estimates on Martingales \citep{peel2013empirical, waudby2024estimating}, U-statistics \citep{peel2010empirical}, PAC-Bayesian bounds \citep{tolstikhin2013pac}, and certain Banach space-valued random variables \citep{martinez2024empirical}. To our knowledge, the present paper gives the first application to weakly dependent variables in a Hilbert space. Other relevant works on learning with non-i.i.d. data include 
\citep{hang2014fast,steinwart2009fast,steinwart2009learning,smale2009online,hang2017bernstein, modha1996minimum, blanchard2019concentration,liu2023wasserstein, alquier2019exponential, abeles2024generalization, abeles2024online, chatterjee2024generalization}. These studies do not consider data-dependent bounds for random variables in Hilbert spaces or focus on operator learning, so they cannot be directly compared with ours. 

Within the context of operator learning, it's notable that \citep{Kostic2022} utilizes the block method from \citep{yu1994rates} to derive estimation bounds for the covariance operator in the domain of transfer operator learning. The concept of a mixing assumption in learning theory is likely first introduced by \citep{yu1994rates}, alongside a technique for method of interlacing block sequences to derive empirical bounds. The authors \citep{modha1996minimum} soon applied these ideas to provide a version of Bernstein's inequality, a scalar version of Theorem \ref{thm:Data-dependent-Bernstein}. Since then, mixing and the method of blocks have been used by numerous authors (\citep{meir2000nonparametric}, \citep{mohri2008rademacher}, \citep{steinwart2009fast}, \citep{agarwal2012generalization}, \citep{shalizi2013predictive}, and others). We follow a similar approach, with the distinction that our data consist of a sequence of vectors $X_{1},X_{2},...$ in a Hilbert space, and we combine Bernstein's inequality with empirical estimates of its variance term.


{\bf Contributions~} In summary our main contributions are: {\bf 1)} We present novel empirical Bernstein inequalities for a sequence of vectors in a Hilbert space; {\bf 2)} We apply these inequalities to derive estimation bounds for the covariance and cross-covariance matrices of the process, showing improvement over recent bounds in the context of stochastic dynamical system and Koopman operator regression (KOR); {\bf 3)} We use our EBI to prove risk bounds for learning stochastic processes, which due to its empirical nature avoids the need for (typically unverifiable) regularity assumptions; {\bf 4)} We present experiments illustrating our theory, and, notably, show that our bounds help understanding generalization in moderate sample-size regimes, and can serve as practical model selection tool in learning dynamical systems.

\section{Theoretical results}\label{sec:theory}
The objective of this section is to bound the error incurred when estimating the mean of a random vector by its average on an observed trajectory. 
We study the norm of the random variable \begin{equation}
\frac{1}{n}\sum_{t=1}^{n}\left( X_{t}{-}\mathbb{E}\left[ X_{t}%
\right] \right), \label{main deviation} \end{equation} where $\mathbf{X}=\left( X_{1},...,X_{n}\right) $ is a vector
of random variables in a separable Hilbert space $\RKHS$, representing the observations along the trajectory.
The principal difficulty is the mutual dependence of the $X_t$. To explain our assumptions to replace independence, we briefly define the $\beta$-mixing coefficients for stochastic processes and describe the method of blocks to approximate a sum of dependent random vectors by a sum of independent block-averages. After that, we state our main results and conclude this section with a sketch of their proofs.
\subsection{Backgrounds on \texorpdfstring{$\beta$} \texorpdfstring{$-$}mixing coefficients and the method of blocks} \label{sec:mixing_blocks}

The key idea for handling dependent data observed from a random, temporal process is that, while subsequent observations may be strongly dependent, they often become approximately independent when separated by sufficiently large time intervals. Such processes tend to forget their distant past. 

To estimate the mean of a random vector in Hilbert space from the
observation of a single trajectory $\mathbf{X}=\left( X_{1},...,X_{n}\right) {\sim} \mu
$ of a largely unknown random process, for $\tau {\in} \mathbb{N}$, we define 
the mixing coefficients \citep{bradley2005basic},
\begin{equation*}
\begin{aligned}    
\beta _{\mu}\left( \tau \right) =\sup_{j\in \mathbb{N}}\sup_{B\in
\Sigma \left( \left[ 1,j\right] \cup \left[ j{+}\tau,\infty \right) \right)
}\big\vert \mu_{[ 1,j] \cup [j +\tau ,\infty)
}(B)\, - \mu _{\left[ 1,j\right]}\times \mu _{\left[ j +\tau
,\infty \right)}(B)\big\vert.
\end{aligned}
\end{equation*}%
Here, $\Sigma \left( \left[ 1,j\right] \cup \left[ j{+}\tau ,\infty \right)
\right) $ is the set of events depending on the $X_{t}$ with $t\in \left[ 1,j%
\right] $ and $t\in \left[ j{+}\tau ,\infty \right) $, the measure $\mu _{%
\left[ 1,j\right] \cup \left[ j{+}\tau ,\infty \right) }$ is the joint
distribution of these variables, and $\mu _{\left[ 1,j\right] }\times \mu _{%
\left[ j{+}\tau ,\infty \right) }$ is the product measure, where events in $%
\Sigma \left( \left[ 1,j\right] \right) $ and $\Sigma \left( \left[ j{+}\tau
,\infty \right) \right) $ are independent (we write $\mu_I$ for the joint distribution of ${X_i:i \in I}$). By definition, this independence
assumption incurs a penalty of $\beta _{\mu}\left( \tau \right) $.
For $m$ events, mutually separated by $\tau $ time increments, the penalty
of assuming them to be independent then increases to $\left( m-1\right)
\beta _{\mu}\left( \tau \right) $. The mixing coefficients are
necessarily non-increasing. Typical assumptions are algebraic mixing, $\beta
_{\mu}\left( \tau \right) \approx \tau ^{-p}$, or exponential mixing $%
\beta _{\mu}\left( \tau \right) \approx \exp \left( -p\tau \right) $
for $p>0$.

Let us assume that $n$ is an even multiple of some $\tau$, that is, $n = 2m\tau$. Following the seminal work of \citep{yu1994rates} we divide
the entire time interval into two sequences of blocks of length $\tau $,
where the blocks in each sequence are mutually separated by $\tau $. Thus
one sequence is $I_{1},I_{2},...,I_{m}$, where each $I_{k}$ has $\tau $
points and the distance between different $I_{k}$ and $I_{l}$ is at least $\tau $. The other sequence $I_{1}^{\prime },I_{2}^{\prime
},..., I_{m}^{\prime }$ has the same properties and fills the gaps left by
the first sequence.  In other word for $\tau \in \mathbb{N}$ and $k \in [m]$ these index sets are $I_{k}=\left\{ 2\left(
k{-}1\right) \tau {+}1,...,\left( 2k{-}1\right) \tau \right\} $ and $I_{k}^{\prime
}=\left\{ \left( 2k{-}1\right) \tau {+}1,...,2k\tau \right\} $.

If we drop the normalizing factor $1/n$, which can always be re-inserted in our bounds, the random variable in (\ref{main deviation}) can then be written as
\begin{equation}
\sum_{i=1}^{n}\left( X_{i}{-}\mathbb{E}\left[ X_{i}\right] \right) \!=\sum_{k {= }1}^{m}(Y_{k}{-}\mathbb{E}\left[ Y_{k}\right])\!+\sum_{k = 1}^{m}(Y_{k}^{\prime
}{-}\mathbb{E}\left[ Y_{k}^{\prime
}\right]),  \label{Decomposition}
\end{equation}%
where the $Y_{1},...,Y_{m}$ and $Y_{1}^{\prime },...,Y_{m}^{\prime }$ are
the "block sums"%
\begin{equation}
Y_{k}=\sum_{i\in I_{k}} X_{i} \text{ and }Y_{k}^{\prime }=\sum_{i\in
I_{k}^{\prime }} X_{i} .
\label{Withinblockaverages}
\end{equation}%
Clearly the $Y_{i}$ are mutually separated by $\tau $ time increments, so
they may be assumed mutually independent at a penalty of $\left( m-1\right)
\beta _{\mu}\left( \tau \right) $, and the same holds for the $%
Y_{i}^{\prime }$. To express these independencies we write $\Pr_{I}$ for the
probability measure $\mu _{I_{1}}\times ...\times \mu _{I_{m}}$ on $\Sigma
\left( I_{1}\cup ...\cup I_{m}\right) $ and $\Pr_{I^{\prime }}$ for $\mu
_{I_{1}^{\prime }}\times ...\times \mu _{I_{m}^{\prime }}$ on $\Sigma \left(
I_{1}^{\prime }\cup ...\cup I_{m}^{\prime }\right) $. Then the above ideas
are summarized by the following lemma, with a detailed proof in Appendix 
\ref{app:mixing}.

\begin{restatable}{lemma}
{BlockingLemma}\label{lem:blocking}
Let $X_{i}$ have values in a normed space $%
\left( \mathcal{X},\left\Vert \cdot \right\Vert \right) $ and let $F,F^{\prime }:%
\mathcal{X}^{n}\rightarrow \mathbb{R}$, where $F$ is $\Sigma \left(
I_{1}\cup ...\cup I_{m}\right) $-measurable, and $F^{\prime }$ is $\Sigma
\left( I_{1}^{\prime }\cup ...\cup I_{m}^{\prime }\right) $-measurable. Then 
\begin{align*}
& \Pr \left\{ \left\Vert \sum_{i=1}^{n}\left( X_{i}-\mathbb{E}%
\left[ X_{i}\right] \right) \right\Vert >F\left( \mathbf{X}\right)
+F^{\prime }\left( \mathbf{X}\right) \right\}  \\
& \leq \Pr_{I}\left\{ \left\Vert \sum_{k =
1}^{m}(Y_{k}-\mathbb{E}\left[ Y_{k}\right])\right\Vert >F\left( \mathbf{X}\right) \right\} \\
&+\Pr_{I^{\prime
}}\left\{ \left\Vert \sum_{k = 1}^{m}(Y_{k}^{\prime
}-\mathbb{E}\left[ Y_{k}^{\prime
}\right])\right\Vert
>F^{\prime }\left( \mathbf{X}\right) \right\} {+}2\left( m{-}1\right) \beta _{\mu}\left( \tau \right) ,
\end{align*}%
where the $Y_{k}$ and $Y_{k}^{\prime }$ are given by (\ref%
{Withinblockaverages}).
\end{restatable}

So the problem of bounding the norm of the dependent sum is reduced to
bounding the norms of two independent sums, albeit with an effective sample
size reduced by a factor of $2\tau $.  The interlaced sequences of blocks are a standard method to port
bounds from the independent to the dependent case. The novelty here is the introduction of $F$ and $F^{\prime }$, needed for our empirical bounds.

For an unknown process, explored only by observation of a single
trajectory, the coefficients $\beta _{\mu}\left( \tau \right) $ are
fixed largely based on plausibility, making $\tau $ very
uncertain. Any bound on the estimation error should therefore depend as little as possible on $\tau $, which determines the effective sample size $%
m=n/\left( 2\tau \right) $. 

\subsection{Bernstein inequalities for vector-valued processes}
\label{sec:Bern_ineq}
Bernstein-type concentration inequalities for functions of $m$ independent variables bound an estimation error by the sum of two terms, a rapidly decreasing term of order $\frac{1}{m}$ and another term, which decreases as $\sqrt{V/m}$, where $V$ is related to the variance of the variables. For the blocking technique, this variance becomes the average of variances of within-block averages, $\frac{1}{\tau}Y_{k}$s and $\frac{1}{\tau}Y'_{k}$s.  Bernstein's inequality can exploit the fact that the correlation of temporally separated variables often decreases orders of magnitude faster than they attain approximate independence. The mixing time $\tau $ then enters mainly in the rapidly decreasing term of order $\tau /n$. This fact has also been pointed out by~\citep{ziemann2024noise}.

For a specified $\tau$ define the set $\Stau \subseteq \left[ n\right] \times \left[ n\right] $  by $
\Stau=\bigcup_{k=1}^{m}\left( I_{k}\times I_{k}\right) \cup \left( I_{k}^{\prime
}\times I_{k}^{\prime }\right)$, where $|\Stau| = 2m\tau^{2}$,  and 
\begin{equation}
\Corrtau{X} = \frac{1}{|\Stau|}\sum_{\left(
t,s\right) \in \Stau}\EE[ \scalarp{X_{t},X_{s}}] {-}\scalarp{\EE[X_{t}], \EE[X_{s}]} 
\end{equation}

The following is our first result.

\begin{restatable}{theorem}
{thmdataBern}\label{thm:Data-dependent-Bernstein}Let $m,\tau \in \mathbb{N}$, $%
n=2m\tau $, and let $\mathbf{X}=\left( X_{1},...,X_{n}\right) $ be a vector
of random variables in a separable Hilbert-space $\RKHS$, satisfying $\left\Vert
X_{t}\right\Vert \leq c$ for all $t$. Let $\delt = \delta{-}2(\frac{n}{2\tau}{-}1)\beta _{\mu}(\tau) > 0$. Then with probability
at least $1{-}\delta $ we have

\begin{equation*}
\begin{aligned}
\!\left\Vert \frac{1}{n}\sum_{t=1}^{n}\left( X_{t}{-}\mathbb{E}\left[ X_{t}%
\right] \right) \right\Vert \leq 
\sqrt{\frac{2\tau\Corrtau{X}}{n}\left(1 {+} 2\ln \frac{2}{\delt} \right)}{+}\frac{8\tau c}{3n}\ln \frac{2}{\delt}.
\end{aligned}
\end{equation*}
\end{restatable}

If the failure probability $\delta$ is fixed, $\tau $ must be sufficiently
large to satisfy $\delta > 2(\frac{n}{2\tau}{-}1)\beta _{\mu}(\tau)$ and result from the specific assumptions on the decay of mixing coefficients.  The variance surrogate, $\Corrtau{X}$, can be bounded by $c^{2}$, so the entire term is at worst of order $ \sqrt{\tau /n}$, but it may become arbitrarily small depending on the joint distribution of the $X_{t}$.

If the law of the process is unknown, then the previous result is not
satisfactory, since in addition to the uncertain but unavoidable, mixing
assumptions we now also need assumptions on the behavior of correlations,
unless we want to return to the worst-case bound. Fortunately $\Corrtau{X}$ can be estimated from the same trajectory and the
estimates can be combined with Theorem \ref{thm:Data-dependent-Bernstein} to give empirical Bernstein-type inequalities, an idea which has
been successfully applied to a variety of problems \citep{audibert2007variance,maurer2009empirical,audibert2009exploration,burgess2020engineered,jin2022policy,tolstikhin2013pac,shivaswamy2010empirical,peel2010empirical,peel2013empirical}.

We give two versions, using a biased variance estimate for general processes and an unbiased estimate for stationary processes. The symmetric structure of $\Stau$ implies, that $\sum_{\left(
t,s\right) \in \Stau}\scalarp{\EE[X_{t}], \EE[X_{s}]} \geq 0$, so that the centered correlations in $\Corrtau{X}$ can be bounded by uncentered ones.
Using this biased estimator 
\begin{equation}
\label{eq:BEcorr}
\BECorrtau{X} = \frac{1}{|\Stau|}\sum_{\left(
t,s\right) \in \Stau}\scalarp{X_{t},X_{s}},
\end{equation} where $|\Stau| = 2m\tau^{2}$, then leads to the following.

\begin{restatable}{theorem}{thmempBernbiased}
\label{thm:empirical-Bernstein-biased}Under the conditions of Theorem %
\ref{thm:Data-dependent-Bernstein} we have for $\delta \in (0,1)$, and $\delt = \delta{-}2(\frac{n}{2\tau}{-}1)\beta _{\mu}(\tau) > 0$ with probability at least $1{-}\delta$ that
\begin{equation*}
\begin{aligned}
\left\Vert \frac{1}{n}\sum_{i}\left( X_{i}{-}\mathbb{E}\left[ X_{i}\right]
\right) \right\Vert \leq \sqrt{\frac{2\tau\BECorrtau{X}}{n}\left( 1{+}2\ln(\frac{4}{\delt})\right)}{+}\frac{32\tau c}{3n}\ln
\frac{4}{\delt}.
\end{aligned}
\end{equation*}
\end{restatable}

If the observed, uncentered correlations decay quickly relative to $\tau $
the first term can be nearly as small as $\sqrt{1/n}$, but no smaller. {Because we have at least $n$ elements in the diagonal of the whole matrix so $\BECorrtau{X}$ can not be smaller than $\min_{t} \left\Vert X_{t} \right\Vert^{2}_{\RKHS}/\tau$}.

We give an improved estimate for stationary processes. A process is called stationary if the joint distributions satisfy $\mu
_{I}=\mu _{I{+}t}$ for any $t \in \mathbb{N}$. In this case, we can estimate the unbiased variance using a u-statistic as follows:
\begin{equation}
\label{eq:Ecorr}
\UECorrtau{X} = \frac{1}{|\Stau|}\left(\!\sum_{\left(
t,s\right) \in \Stau}\!\scalarp{X_{t},X_{s}} {-}\frac{1}{m{-}1}\sum_{\left(
t,s\right) \in \Stautilde} \!\scalarp{X_{t},X_{s}}\right)
\end{equation}  where $|\Stau| = 2m\tau^{2}$, and $|\Stautilde| = 2m(m-1)\tau^2$. Here \begin{equation*}
\Stautilde=\bigcup_{k\neq l:k,l\in \left[ m\right] }\left( I_{k}\times
I_{l}\right) \cup \left( I_{k}^{\prime }\times I_{l}^{\prime }\right) .
\end{equation*}

\begin{restatable}{theorem}{thmempBernunbiased}
\label{thm:empirical-Bernstein-unbiased}Under the conditions of Theorem %
\ref{thm:Data-dependent-Bernstein}, if the process is also
stationary, we have for $\delta \in \left(
0,2/e\right) $ and $\delt = \delta{-}2(\frac{n}{2\tau}{-}1)\beta _{\mu}(\tau) > 0$ with probability at least $1{-}\delta$ that
\begin{equation*}
\begin{aligned}
\left\Vert \frac{1}{n}\sum_{i}\left( X_{i}{-}\mathbb{E}\left[ X_{i}\right]
\right) \right\Vert \leq \sqrt{\frac{2\tau\UECorrtau{X}}{n}\left( 1{+}2\ln(
\frac{4}{\delt})\right)} {+}\frac{22\tau c}{n}\ln(
\frac{4}{\delt}).
\end{aligned}
\end{equation*}
\end{restatable}

Think of $\left[ n\right] \times \left[ n\right] $ as the surface of a chessboard with fields of side length $\tau $. Then $m=4$, and $\Stau$ is the union of the squares on the white diagonal, and $\Stautilde$ is the union of all other white squares (see Figure \ref{fig:blocks}). If $\left( s,t\right) \in \Stau$ then $s$ and $t$ are no more than $\tau {-}1$ apart, while for $\left( s,t\right) \in \Stautilde$ they are at least $\tau $ apart.
\begin{figure}[htbp]
    \centering
\includegraphics[width=0.8\linewidth]{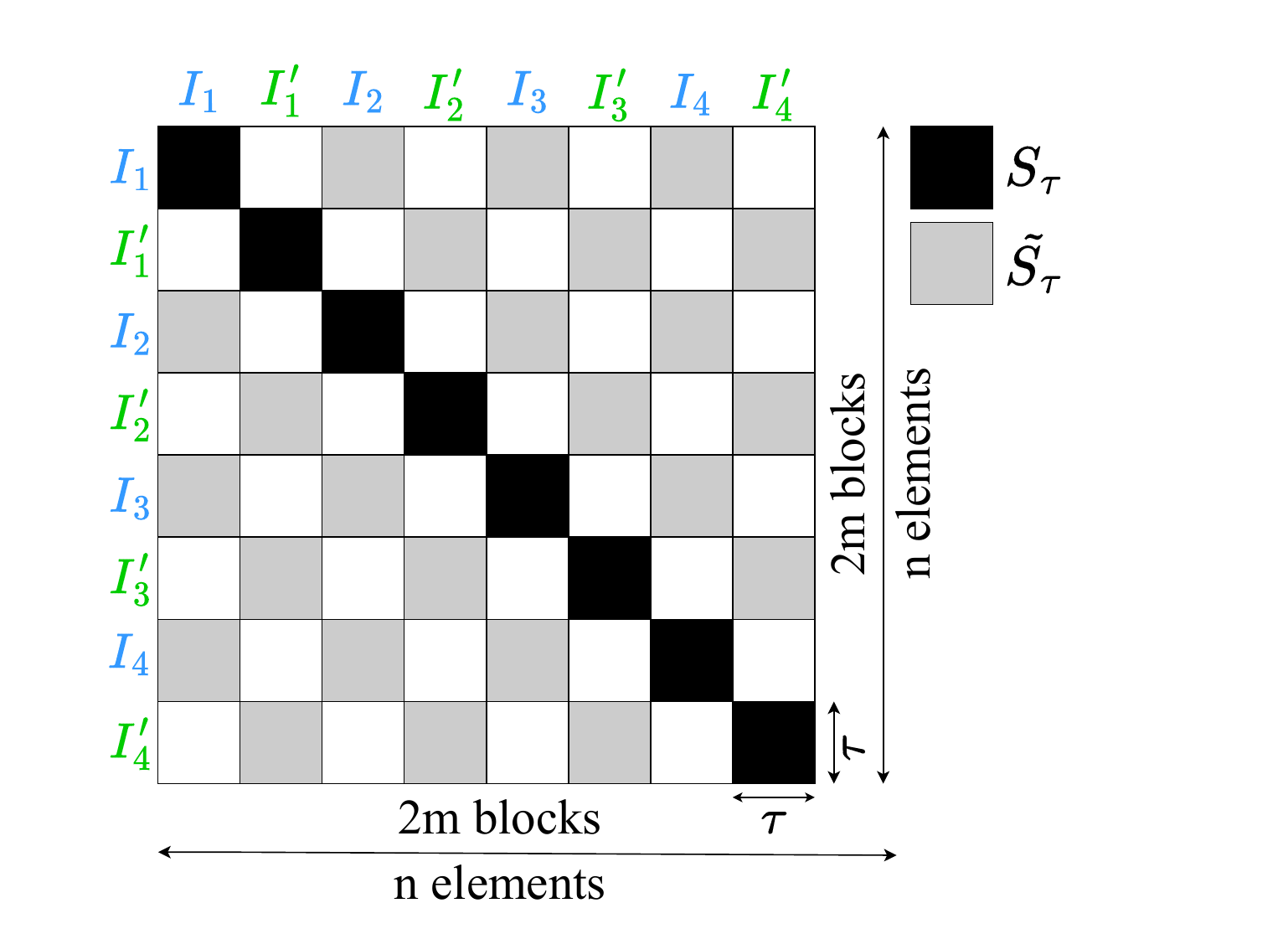}
    \caption{The blocks belonging to $\Stau$ and $\Stautilde$.}
    \label{fig:blocks}
\end{figure}
Notice that the passage to empirical bounds incurs an additional constant
factors only in the bound for the last term.
\subsection{Proofs}
\textbf{\textit{Sketch of proof}} (detailed in the supplementary file). As explained in Section \ref{sec:mixing_blocks} the method of blocks provides a tool to port bounds from the independent to the dependent case, so we first establish inequalities for independent data.

A concentration inequality of \citep[Theorem 3.8]{McDiarmid98} quite easily leads to the
following Bernstein-type inequality for independent, centered vectors $Y_{i}$
in Hilbert space, satisfying $\left\Vert Y_{i}-\mathbb{E}%
 [Y_{i}]\right\Vert \leq c$.%
\begin{equation}
\begin{aligned}
\Pr \biggl\{ \left\Vert \sum_{k}Y_{k}-\mathbb{E}\left[ Y_{k}\right]
\right\Vert >\sqrt{\sum_{k}\mathbb{E}\left\Vert Y_{k}-\mathbb{E}\left[ Y_{k}%
\right] \right\Vert ^{2}}\left(1+\sqrt{2\ln \left( 1/\delta \right) }%
\right) +\frac{4c}{3}\ln \left( 1/\delta \right) \biggl\} <\delta .
\end{aligned}
\label{template thm 1}
\end{equation}

The additional 1 in $1+\sqrt{2\ln{(\frac{1}{\delta})}}$ arises from a bound of the expected norm by the root of the variance. To obtain empirical bounds, we want to combine this with
estimates of the variance term. Without additional assumptions, we obtain
from a concentration inequality for real-valued functions that

\begin{equation*}
\begin{aligned}
\sqrt{\sum_{k}\mathbb{E}\left\Vert Y_{k}-\mathbb{E}\left[ Y_{k}\right]
\right\Vert ^{2}}&\leq \sqrt{\sum_{k}\left\Vert Y_{k}\right\Vert ^{2}}+c\sqrt{2\ln \left(
1/\delta \right) }.
\end{aligned}
\end{equation*}

Combining this with (\ref{template thm 1}) in a union bound gives, with probability at least $%
1-\delta $ 
\begin{equation*}
\begin{aligned}
\Pr \biggl\{ \left\Vert \sum_{k}Y_{k}-\mathbb{E}\left[ Y_{k}\right]
\right\Vert >\sqrt{\sum_{k}\left\Vert Y_{i}\right\Vert ^{2}}\left( 1+\sqrt{%
2\ln \left( 2/\delta \right) }\right) +\frac{16c}{3}\ln \left( 2/\delta
\right) \biggl\} <\delta .  \label{template thm 2}
\end{aligned}
\end{equation*}%
This is our independent template for Theorem 2.

To obtain Theorem 3, we can use stationarity of the process, which means that
the $Y_{k}$ (or the $Y_{k}^{\prime }$) can now be assumed to have identical
distribution. In this case, a slightly more involved argument gives the
estimate
\begin{equation*}
\begin{aligned}    
\sqrt{\sum_{k}\mathbb{E}\left\Vert Y_{k}-\mathbb{E}\left[ Y_{k}\right]
\right\Vert ^{2}}\leq \sqrt{\frac{1}{2\left( m-1\right) }\sum_{k,l:k\neq
l}\left\Vert Y_{k}-Y_{l}\right\Vert ^{2}}+4c\sqrt{2\ln \left( 1/\delta
\right) }.
\end{aligned}
\end{equation*}

Again, a union bound with (\ref{template thm 1}) gives the independent template for Theorem 3. 

Now we port these inequalities to the dependent case using Lemma \ref{lem:blocking}. For Theorem 1, we define
\begin{equation*}
\begin{aligned}
F\left( \mathbf{X}\right)  =\sqrt{\sum_{k=1}^{m}\mathbb{E}\left\Vert
Y_{k}-\mathbb{E}\left[ Y_{k}%
\right]\right\Vert ^{2}}\left( 1+\sqrt{2\ln \left( 2/\delta \right) }\right) +\frac{4\tau c}{3}\ln \left( 2/\delta \right)  
\end{aligned}
\end{equation*}
with $F^{\prime }\left( \mathbf{X}\right)$ defined analogously, replacing $\overline{Y}_{k}$ by $Y^\prime_k$. 
Substitution of these definitions in the body of Lemma \ref{lem:blocking} , using  (\ref{template thm 1}) to bound the two independent probabilities, 
some algebraic simplifications and reinsertion of the overall factor of $1/n=1/\left(
2m\tau \right) $\ give Theorem 1.

Above $F$ and $F^{\prime }$ were simply constant functions. This is
different for the empirical bounds, where they contain the variance
estimates. To obtain Theorem 2, we let  
\[
F\left( \mathbf{X}\right) =\sqrt{\sum_{k=1}^{m}\left\Vert Y_{k}\right\Vert
^{2}}\left( 1+\sqrt{2\ln \left( 4/\delta \right) }\right) +\frac{16c\tau}{3}\ln
\left( 4/\delta \right) ,
\]%
which is $\Sigma \left( I_{1}\cup ...\cup I_{m}\right) $-measurable, and
replace $Y_{k}$ by $Y_{k}^{\prime }$ for the analogous definition of $%
F^{\prime }\left( \mathbf{X}\right) $. Then, using Lemma \ref{lem:blocking} and (%
\ref{template thm 2}), unraveling the definitions, and some simplifications give Theorem 2. Similarly, 
\begin{equation*}
\begin{aligned}
F\left( \mathbf{X}\right) =\sqrt{\frac{1}{2\left( m{-}1\right) }\sum_{k,l:k\neq
l}\left\Vert Y_{k}-Y_{l}\right\Vert ^{2}}\left( 1{+}\sqrt{2\ln \left( 4/\delta
\right) }\right) +11c\tau\ln \left( 4/\delta \right),
\end{aligned}
\end{equation*}
and the corresponding $F^{\prime }\left( \mathbf{X}\right) $ give Theorem 3.

\section{Applications}
In this section, we address two important fields of study where our bound can lead to new advances. First, we note that the concentration inequality is naturally linked to covariance estimation, which plays an important role in machine learning and statistics \citep{markowitz1952portfolio,vonstorch1999principal,schäfer2005shrinkage, mollenhauer2022kernel}. 
The second application concerns data-driven dynamical systems and, in particular, transfer operators, which are also widely used in science and engineering~\cite[see e.g.][and references therein]{Brunton2022,tuckerman2023statistical}.

{\bf Covariance estimation~}
We apply our results to the estimation of covariance operators on the Reproducing Kernel Hilbert Space (RKHS) $\RKHS$ with an associated kernel function $k:\X\times\X \to \R$. Letting $\phi:\X \to \RKHS$ be a {\em feature map} such that $k(x,x^\prime) {=} \scalarp{\phi(x), \phi(x^\prime)}$ for all $x, x^\prime \in \X$, we aim to bound the Hilbert-Schmidt norm  estimation error. Hence, in this setting, the observed vectors are operators, and to apply our bounds, one needs to replace $X_{t}$ by the rank-one operator $Y_t {=}\phi(X_{t})\otimes \phi(X_{t})$. Then, $\BECorrtau{Y}$ and $\UECorrtau{Y}$ can be easily computed by using entries of the kernel matrix.
If the process is not stationary,
we can only estimate the ergodic average of covariance operators, that is 
$\left( 1/n\right) \sum_{t}\mathbb{E}\left[  \phi(X_{t})\otimes
\phi(X_{t}) \right] $, which becomes $C{=}\mathbb{E}\left[
\phi(X_{1})\otimes \phi(X_{1}) \right] $ for stationary processes. 
The transcription of the previous results fortunately, is affected by simply replacing all
$\Cxy{=} 
\left( 1/n\right) \sum_{t}\mathbb{E}%
\left[  \phi(X_{t})\otimes \phi(X_{t+1}) \right] $ because the
inner products change; the block variables are then only separated by $\tau
-1$ instead of $\tau $, and we also need to observe one more point on the
trajectory.

To demonstrate the application of our empirical Bernstein's inequalities to estimate the covariance and the improvements it introduces, we first adapt Pinelis and Sakhanenko's concentration inequality for random variables in a separable Hilbert space (see \citep[][Proposition 2]{caponnetto2007}) to trajectory data, using the method of blocks and $\beta$-mixing, as it was suggested in \citep{kostic2023sharp}, see Appendix \ref{app:cov_est}

Assuming that the data is sampled from the stationary distribution, our bounds in Theorems   \ref{thm:empirical-Bernstein-biased} and \ref{thm:empirical-Bernstein-unbiased} apply. Compared to worst-case bounds, the improvement lies in the correlation factors $\UECorrtau{X}$ and $\BECorrtau{X}$ 
in \eqref{eq:Ecorr} and \eqref{eq:BEcorr}, respectively, affecting the slow term in the bound. Before we illustrate this improvement empirically, we discuss a related application. 

{\bf Learning dynamical systems~} 
Recent advances in the statistical theory for Koopman operator learning have highlighted the significant impact of covariance estimation on the ability to generalize when forecasting and interpreting dynamical systems from data-driven models~\citep{Kostic2022, Philipp2024}. While statistical learning theory has been thoroughly developed for kernel methods~\citep{kostic2023sharp,kostic2024consistent} under $\beta$-mixing assumptions, an important gap remains between the learning rates and generalization bounds, and the practical performance of the methods. This is particularly interesting when comparing recent deep learning advances with kernel methods~\citep{kostic2024dpnets}. Here, we briefly review the transfer operator learning and then present novel contributions to this field based on our EBI.

For a {\em time-homogeneous} Markov chain with an invariant (stationary) distribution $\im$ the (stochastic) \textit{Koopman operator} $\Koop \colon \Lii\to\Lii$  is given by \\
\(
\mbox{ }\quad[\Koop f](x) {:=} \displaystyle{\int_{\X}} p(x, dy)f(y) {= }\mathbb{E}\left[f(X_{t {+} 1}) \middle | X_{t} = x\right],
\)
where $\quad f\in\Lii,x\in\X$.

In many practical cases, $\Koop$ is unknown, but data from one or multiple system trajectories are available. A framework for operator regression learning was introduced in~\citep{Kostic2022} to estimate the Koopman operator on $\Lii$ within an RKHS using an associated feature map $\fH:\X \to \RKHS$. In this vector-valued regression, the risk functional is defined as 
\(
    \Risk(G) = \textstyle{\EE_{X\sim\pi,X^+\sim p(\cdot\vert X))} \norm{\phi(Y) - G^{*}\phi(X)}^{2}}_\RKHS,
\)
and the task is to learn $\Koop$ by minimizing the risk over some class of operators $G\colon{\RKHS\to\RKHS}$ using a dataset of consecutive states  $\Data := (x_{i}, x^+_{i})_{i = 1}^{n}$. Typical scenario is to obtain the states from a single trajectory of the process after reaching the equilibrium distribution, that is $X_0\sim \pi$, $X^+_i\equiv X_{i+1} \sim p(\cdot\,\vert\,X_i)$, $i=2,\ldots,n$, and the popular estimator in this setting is the Reduced Rank Regression (RRR) one $\ERRR$ obtained by minimizing regularized empirical risk \(
    {\ERisk}_\lambda(\Estim) := \textstyle{\tfrac{1}{n}\sum_{i \in[n]}} \norm{\phi(x^+_{i}) - \Estim^{*}\phi(x_i)}^{2}_{\RKHS} {+} \lambda\|\Estim\|^2_{\text{HS}},
\) over operators $\Estim$ of rank at most $r$, that is $\ERRR = \ECreg^{-1/2}\SVDr{\ECreg^{-1/2}\ECxy}$ is computed via $r$-truncated SVD $\SVDr{\cdot}$ and  {\em input} and {\em cross} empirical covariances $\ECx \,= \,\textstyle{\tfrac{1}{n}\sum_{i \in[n]}}\, \phi(x_{i}){\otimes} \phi(x_{i})$
and
$\ECxy\,= \,\textstyle{\tfrac{1}{n}\sum_{i \in[n]}}\, \phi(x_{i}){\otimes }\phi(x^+_{i})$, respectively, while
$\ECx_{\reg} = \ECx + \reg I_{\RKHS}$, c.f. \citep{Kostic2022,kostic2023sharp}.

The recent works \citep{Li2022, kostic2023sharp} on the mini-max optimal learning rate for Koopman operator regression in i.i.d. setting crucially rely on Pinelis and Sakhanenko's Inequality.  
As observed above, applying the method of blocks and $\beta$-mixing extends the i.i.d. analysis of transfer operator regression to realistic scenarios of learning from data trajectories of a stationary process. However, the overall approach relies on several (in practice unverifiable) assumptions and results in an unfavorable impact of mixing on the learning rates.

In the following, we present a novel risk bound for the reduced-rank Tikhonov estimator that circumvents the need for regularity assumptions, in \citep{Li2022,kostic2023sharp,kostic2024consistent,kostic2024learning}.
 
\begin{restatable}{theorem}{TikhonovRiskBound}
\label{thm:tikhonov_risk_bound}
Let $\mathbf{X}=(X_{t})_{t=1}^{n}$ be a stationary Markov chain with distribution $\mu$, and the risk definitions as above noting that $\pi=\mu_{1}$. Denote $\EEstim_{r, \lambda}$ be a minimizer of reduced-rank Tikhonov regularized empirical risk
and $\mathbf{Y} = (\fH(X_{t})\otimes\fH(X_{t}))_{t=1}^{n}$, $\mathbf{Z} = (\fH(X_{t})\otimes\fH(X_{t+1}))_{t=1}^{n}$, and $\mathbf{W} = (\|\fH(X_{t})\|^2)_{t=1}^{n}$. Assume $n =2m\tau$, and exists $\bcon\,>\,0$  such that $\norm{\fH(X_t)}^2{\leq }\bcon$ a.s. for all $t$. Let $\delta {\geq} 0$ and assume $\widehat{\delta}_{\mu}(\tau,\lambda) := 0.5\,\delta/\|\EEstim_{r, \lambda}\|{-}2(\frac{n}{2\tau}{-}1)\beta _{\mu}(\tau) > 0$. \\
Then, with probability at least $1 {-} \delta$ we have for every $\EEstim_{r, \lambda}$ such that $\|\EEstim_{r, \lambda}\|_{HS} {\geq} 1$
\begin{equation*}\label{eq:tikhonov_bound}
\begin{aligned}
|\Risk(\EEstim_{r, \lambda}){-}\ERisk(\EEstim_{r, \lambda})| &{\leq }
\frac{128\bcon\tau\|\EEstim_{r, \lambda}\|(\sqrt{r}{+} \|\EEstim_{r, \lambda}\|)}{3n}\ln{\frac{12}{\widehat{\delta}_{\mu}(\tau,\lambda)}}{+}\frac{14\bcon\tau^2
}{3n-2\tau
}\ln{\frac{12}{\widehat{\delta}_{\mu}(\tau,\lambda)}} \\
&{+} \sqrt{\frac{32\|\EEstim_{r, \lambda}\|^4\BECorrtau{Y}\tau}{n}\left(1
{+}2\ln\frac{12}{\widehat{\delta}_{\mu}(\tau,\lambda)}\right)} 
{+} \sqrt{\frac{2\overline{V}_{\tau}(\mathbf{W})\tau}{n}\ln{\frac{12}{\widehat{\delta}_{\mu}(\tau,\lambda)}}}
\\
&{+} \sqrt{\frac{8r\|\EEstim_{r, \lambda}\|^2\BECorrtau{Z}\tau}{n}\left(1{+}2\ln\frac{12}{\widehat{\delta}_{\mu}(\tau,\lambda)}\right)} ,
\end{aligned}
\end{equation*}

where $\overline{V}_{\tau}(\mathbf{W}) = \frac{1}{m(m-1)\tau^2} \sum_{1\leq i< j\leq m} (\overline{W}_{i} - \overline{W}_{j})^2 + (\overline{W'}_{i} - \overline{W'}_{j})^2$, $\overline{W}_{j}=\sum_{i\in I_{j}}W_i$ and $\overline{W'}_{j}=\sum_{i\in I_{j}^{\prime }}W_i$. $\BECorrtau{\mathbf{Y}}$ and $\BECorrtau{\mathbf{Z}}$ were defined before.

\end{restatable}

{We remark that the only non-computable part in the bound is the mixing coefficient in the log terms, while the correlation coefficients weigh the powers of the estimator's norm. This is particularly interesting because the bound holds in the non-asymptotic regime, and even with moderate sample sizes, it reveals the impact of hyperparameters, namely the ridge parameter $\reg > 0$ and rank parameter $r$, on the risk concentration. We illustrate this feature in the practical problem of model selection when learning molecular dynamics. Further, note that our approach can easily relax the restrictive stationarity assumptions at the cost of an additive term quantifying the distance of the initial distribution from the equilibrium one, see Appendix \ref{app:LDS_theory} for a discussion. This is crucial for scenarios where data is collected out of equilibrium, a challenge not easily addressed by approaches like \citep{Kostic2022, kostic2024consistent}.}

\section{Experiments}\label{sec:exp} 
In this section, we showcase the improvements of our empirical Bernstein inequality for covariance estimation and for learning dynamical systems with moderate sample sizes.

\begin{figure}[h!]
    \centering
\includegraphics[width=\textwidth]{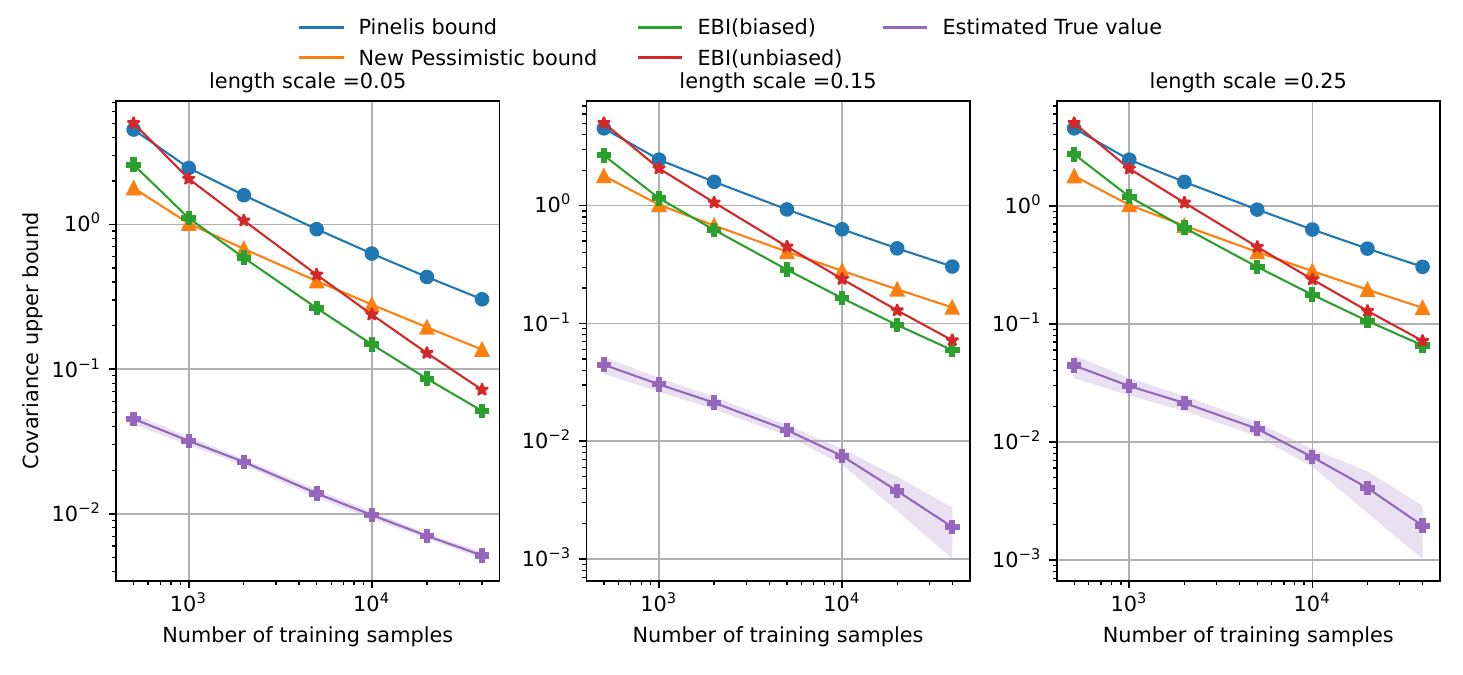}
\vspace{-.3truecm}
    \caption{Covariance upper bound as a function of the number of training points for three different length scales of Gaussian kernel in logarithmic scale. The failure probability is assumed to be 0.05 and the plots have been averaged over 30 independent simulations.}
    \label{fig:OU_cov_bounds}
\end{figure}

{\bf Covariance estimation using samples from Ornstein–Uhlenbeck process~~} In the first experiment, we illustrate the quantitative improvement of our EBI in comparison to its pessimistic non-empirical version, noting that a slightly different form has been used in \citep{Kostic2022}. 
We use the proposed empirical inequalities to determine the concentration of the covariance operator in the Hilbert-Schmidt norm. The usefulness of the new bounds can be particularly exploited when the process decorrelates much faster than it attains independence. To demonstrate this, we use 1D equidistant sampling of the Ornstein–Uhlenbeck process, obtained by integrating 
\(X_{t} = e^{-1}X_{t-1} {+} \sqrt{1-e^{-2}}\,\epsilon_t,\)
where $\{\epsilon_t\}_{t\geq 1}$ are i.i.d. samples from the standard Gaussian distribution. For this process, it is well-known~\citep{Pavliotis2014} that the invariant distribution, $\im$, coincides $\mathcal{N}(0,1)$. One point that makes the use of this as a toy example beneficial is the fact that this process belongs to exponential mixing processes $%
\beta _{\mu}\left( \tau \right) \approx \exp \left( -p\tau \right) $ where $p$ is the gap between first and second eigenvalues of its transfer operator, which is also known $p = 1 {-} 1/e$. We apply a Gaussian kernel with different length scales to map our data to the RKHS with the corresponding feature map, $\fH(x_{t})$. \\
Initially, we set a probability threshold for the failure of the inequalities. Subsequently, with this fixed failure probability in mind, we determined the appropriate mixing time $\tau$ for a given sample size $n$, defined as the smallest value satisfying $\delt$. We opted for $\tau$ due to its optimality since we observed through experimental validation that there is a consistent monotonic increase in the relationship between the empirical bounds and $\tau$, across various training set sizes and different failure probabilities. For further details, please refer to Appendix \ref{app:cov_est_exp_res}
In Figure \ref{fig:OU_cov_bounds}, we plotted empirical upper bounds for covariance estimation for different numbers of training points across three different choices of length scales over 50 independent simulations. We compared the new data-dependent upper bounds with the pessimistic bounds obtained using Pinelis and Sakhanenko's and Theorem 4 concentration inequality. The details of computing the true error value are explained in Appendix \ref{app:cov_est_true_error}. 

Figure \ref{fig:OU_cov_bounds} shows that one can significantly overestimate when using classical Bernstein-type inequalities pessimistically, especially as the number of training points grows. In other words, the slow term in the classical Bernstein inequalities may not be too slow, especially for processes where elements decorrelate rapidly. Importantly, notice that the slopes of the EBI bounds indicate a faster rate $\approx 1/n$ for the moderate sample size regime than the rate $1/\sqrt{n}$ which is asymptotically optimal.

As noted above, the sample covariance operator matches the square of the kernel matrix. For a better understanding, we can examine the extreme cases of the two variance proxy estimates in Appendix \ref{app:cov_est_ext_exp}.

\begin{figure}[h!]
    \centering
    \begin{subfigure}[b]{0.9\textwidth} 
        \centering
\includegraphics[width=\linewidth]{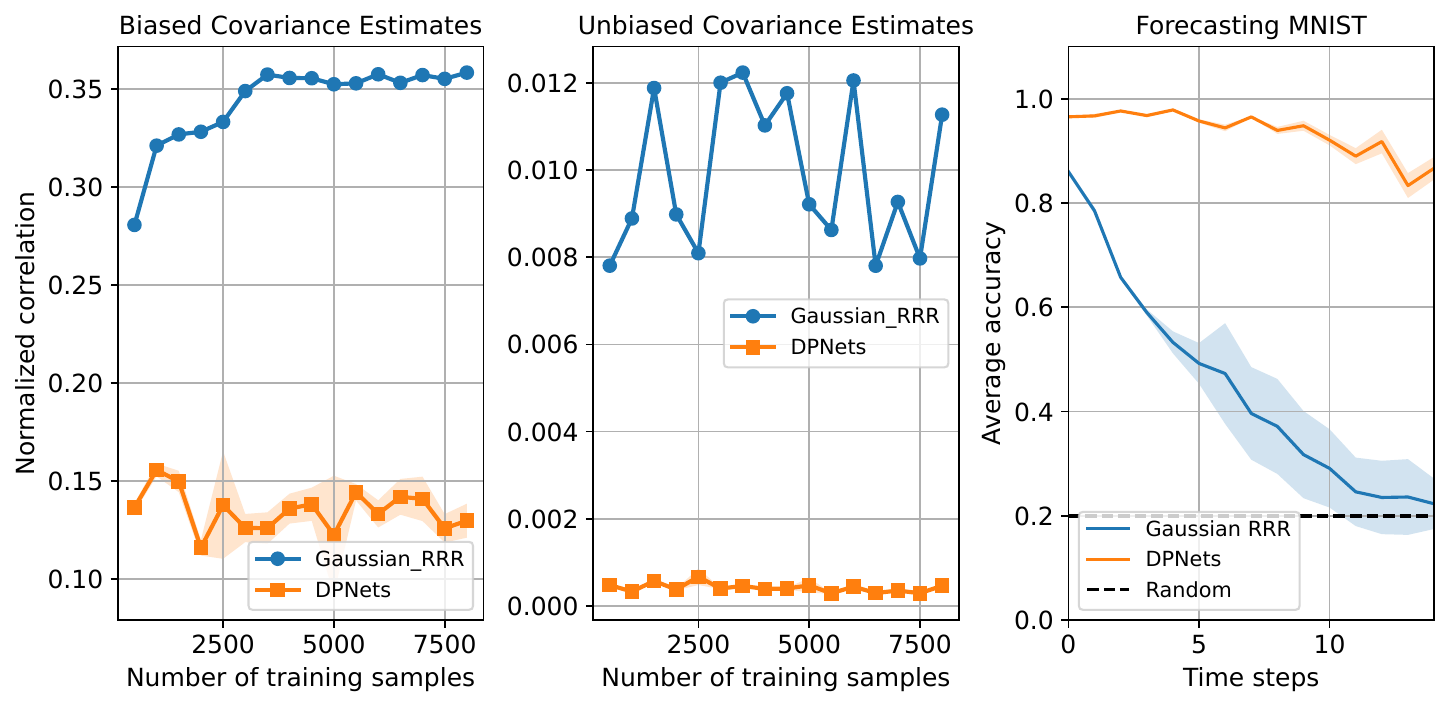}
    \end{subfigure}
    \vspace{0.2cm} 
    \begin{subfigure}[b]{0.9\textwidth} 
        \centering        \includegraphics[width=\linewidth]{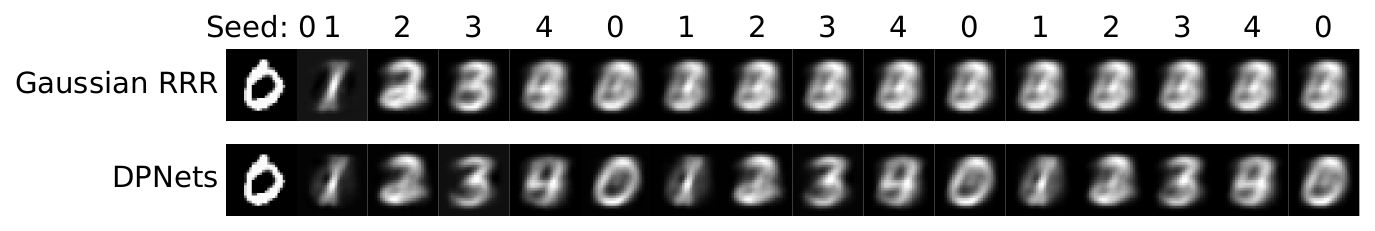}
    \end{subfigure}
    \vspace{-.3truecm}
    \caption{Performance evaluation of rank-5 RRR estimators using Gaussian and DPNet kernels on MNIST with $\eta = 0.1$: normalized correlations and forecast accuracy. }
    \label{fig:MNIST_main}
\end{figure}

{\bf Noisy ordered MNIST~~} Now we design an experiment involving a stochastic process with $K$ states, $S_{1}, S_{2}, \ldots, S_{K}$, corresponding to each integer class of the MNIST dataset. At each time step of a trajectory, we 
sample an MNIST image from that class, with the classes appearing in an ordered sequence, subject to a small perturbation probability. More precisely, the probability of transitioning from $S_{i}$ to $S_{(i{+}1)\,{\rm mod}\,{K}}$ is $1 - \eta$, while the transition to any other state occurs uniformly at random with probability $\eta$. 
Notice the invariant distribution $\im$ is uniform in the above setting. Moreover, the mixing time is $\tau \geq \eta^{-1} \ln \left(1/\epsilon \right)$ for a distance $\epsilon$ from the invariant distribution, $\im$ \citep{levin2017markov}. This provides an estimate of the mixing coefficients $\beta_{\mu} \left( \tau \right) \approx \exp \left(-\eta \tau \right)$. \\
The purpose of this experiment is to emphasize the importance of selecting an appropriate representation for learning the dynamics and to demonstrate different behaviors in terms of the decorrelation rate of the representation. To this end, we selected the first 5 classes of the MNIST dataset as training data points and we compared the Reduced Rank Regression (RRR) estimator in \citep{Kostic2022} with rank $5$, using two different kernels. The first representation is the Gaussian kernel, which is known to be a universal kernel. The second kernel is linear in the space parametrized by a neural network $\phi_{\boldsymbol{\theta}} \in \R^5$ trained according to Deep Projection Neural Network (DPNet) \citep{kostic2024dpnets}, which is designed to minimize the representation error for the operator regression task by minimizing the empirical risk. 
While for the Gaussian kernel, We performed hyperparameter tuning on validation data points, resulting in a length scale of $784$ and a regularization parameter of $10^{-7}$, for DPNet we use a CNN architecture; see Appendix \ref{app:MNIST} for more information. \\
We trained the two transfer operator estimators on different samples. Figure \ref{fig:MNIST_main} demonstrates that the normalized correlations, $\BECorrtau{X}$, and $\UECorrtau{X}$, are effective estimators of the generalization performance of the learned models. We applied min-max normalization to the kernel matrices. The rightmost figure in the first row shows the average accuracy of all forecasts of test points for each model, varying the number of forecasting steps. In this process, the model predicts an image, and we use an oracle CNN to predict the label. We then match this predicted label with the true label and compute the accuracy accordingly. In the second row, we illustrated how this forecasting works for a single sample of the test set. As expected, only the forecasts using the DPNet kernel maintained a sharp (and correct) shape for the predicted digits over a long horizon. In contrast, the Gaussian kernels were less effective in capturing visual structures, and their predictions quickly lost resemblance to digits. \\
The results indicate that models with superior representations tend to perform better and exhibit better generalization. We repeated the procedure with two different values of $\eta = 0.2$ and $0.05$. The results can be found in Appendix C.2. 

\textbf{EBI-based model selection}
In this experiment, we demonstrate that minimizing the bound Theorem \ref{thm:tikhonov_risk_bound} can serve as an effective criterion for selecting Koopman models. We applied this approach to a realistic simulation of the small molecule Alanine Dipeptide, c.f.~\citep{wehmeyer2018time}. Sixteen RRR estimators, each associated with a different kernel's length scale and/or ridge regularization parameter, were trained, and their forecasting RMSE was assessed using 2000 initial conditions from a test dataset. In this procedure, for the beta-mixing coefficient, we used an estimate via the second eigenvalue of the Koopman operator based on the independent studies \citep{ deepTICA}. Figure \ref{fig:ala_hpo} presents these errors, highlighting the model with the lowest empirical risk bound (as defined in Theorem \ref{thm:tikhonov_risk_bound}), which is one of the best estimators.
\begin{figure}[t!]
   \centering
\includegraphics[width=0.6\columnwidth]{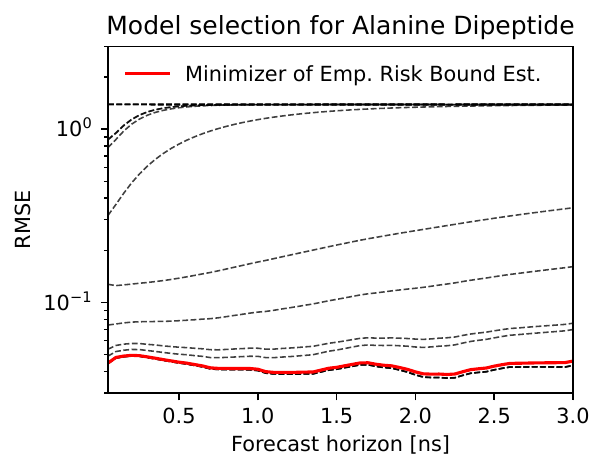}
\vspace{-.35truecm}
{\caption{Forecasting RMSE on the Alanine Dipeptide dataset for 16 different RRR estimators, each corresponding to a different kernel's length scale and/or regularization parameter, which shows how the best model, according to the empirical risk bound metric, also attains on of the best forecasting performances.}
\label{fig:ala_hpo}}
\end{figure}

\section{Conclusion}\label{sec:concl}

Motivated by recent progress on stochastic dynamical systems and the fact that, in realistic scenarios therein, data are neither i.i.d. nor in the stationary regime, we derived empirical Bernstein inequalities for a general class of stochastic processes in Hilbert space. Bernstein inequality is a key component in previous studies of learning dynamical systems \citep{Kostic2022, kostic2024consistent} and we showed that one can retain the guarantees in these papers while relaxing the restrictive assumptions on the data sequence. Notably, our inequalities expose the effect of correlations in improving the bounds. Further, we developed an alternative approach to risk bounds of reduced rank regression, which is less restrictive and more practical than the existing ones, as also illustrated numerically. A limitation of our work is that our bounds in their current form are still not adapted to exploit the effective dimension of the RKHS in deriving the minimax rate and theoretically proving the impact of the correlation coefficients on the excess risk. Future work could address this question as well as investigate the application of our bound to vector-valued problems in the context of partial differential equations.

\bibliographystyle{apalike}
{
\bibliography{bibliography}
}
\newpage
\appendix

\begin{center}
{\Large \bf Supplementary Material} 
\end{center}

\section{Hilbert space-valued concentration for dependent, non-stationary sequences.}

First, we give a version of Bernstein's inequality for independent vector-valued random variables. Then we specialize to the dependent case under $%
\beta $-mixing assumptions.

\subsection{Vector valued concentration for independent variables}

\begin{theorem}
\label{thm:vector-concentration}Suppose that the $X_{i}$ are $m$ independent
mean zero random variables with values in a Hilbert-space $H$, satisfying $%
\left\Vert X_{i}\right\Vert \leq c$. Then for $\delta >0$%
\begin{equation*}
\Pr \left\{ \left\Vert \sum_{i}X_{i}\right\Vert >\sqrt{\sum_{i}\mathbb{E}%
\left\Vert X_{i}\right\Vert ^{2}}\left( 1{+}\sqrt{2\ln \left( 1/\delta \right) 
}\right) {+}\frac{4c}{3}\ln \left( 1/\delta \right) \right\} <\delta .
\end{equation*}
\end{theorem}

A well-known bound of Pinelis and Sakhanenko \citep[Proposition
2]{caponnetto2007}, when specialized to bounded vectors, would be 
\begin{equation*}
2\sqrt{\sum_{i}\mathbb{E}\left\Vert X_{i}\right\Vert ^{2}}\ln \left( \frac{2%
}{\delta }\right) {+}4c\ln \left( \frac{2}{\delta }\right) ,
\end{equation*}
which is slightly worse, not only with respect to constants but also in its
dependence on the confidence parameter $\delta $.

To prove Theorem \ref{thm:vector-concentration} we use the version of
Bernstein's inequality is below. For $\mathbf{z}=\left( z_{1},...,z_{m}\right)
\in \mathcal{X}^{m}$, $k\in \left[ m\right] $ and $y\in \mathcal{X}$ define
the substitution $S_{y}^{k}\mathbf{z}=\left(
z_{1},...,z_{k{-}1},y,z_{k{+}1},...,z_{m}\right) $.

\begin{theorem}
\label{thm:Bernstein-inequality} (see \citep{McDiarmid98}, Theorem 3.8 or 
\citep{maurer2012thermodynamics}, Theorem 11) Let $\mathbf{X}=\left(
X_{1},...,X_{m}\right) $ be a vector of independent random variables, with
values in $\mathcal{X}$ and $f:\mathcal{X}^{m}\rightarrow \mathbb{R}$.
Assume that $\forall k\in \left[ m\right] ,$ $\mathbf{x}\in \mathcal{X}^{m}$%
, $f\left( \mathbf{x}\right) {-}\mathbb{E}\left[ f\left( S_{X_{k}}^{k}\mathbf{x%
}\right) \right] \leq b$. Let denote 
\begin{equation*}
V=\frac{1}{2}\sup_{\mathbf{x}\in \mathcal{X}^{m}}\sum_{k=1}^{m}\mathbb{E}%
\left[ \left( f\left( S_{X_{k}}^{k}\mathbf{x}\right) {-}f\left(
S_{X_{k}^{\prime }}^{k}\mathbf{x}\right) \right) ^{2}\right] .
\end{equation*}%
Then for $t>0$ 
\begin{equation*}
\Pr \left\{ f\left( \mathbf{X}\right) {-}\mathbb{E}\left[ f\left( \mathbf{X}%
\right) \right] >t\right\} \leq \exp \left( \frac{-t^{2}}{2V{+}2bt/3}\right) .
\end{equation*}
\end{theorem}

The quantity $V$ is the "maximal sum of conditional variances" (\citep%
{McDiarmid98}): we compute the variance of $f$ in the $k$-th variable (here $%
X_{k}^{\prime }$ is an independent copy of $X_{k})$, holding all the other
variables of $\mathbf{x}$ fixed, then we sum over $k$ and finally take the
supremum in $\mathbf{x}$.

\begin{proof}[Proof of Theorem \protect\ref{thm:vector-concentration}]
We will apply Theorem \ref{thm:Bernstein-inequality} with $\X = \{ x \in \RKHS : \|x\| \leq c\}$. Define $f\left( \mathbf{x}\right) =\left\Vert \sum_{i}x_{i}\right\Vert $ and
note that for $y,y^{\prime }\in \RKHS$, $f\left( S_{y}^{k}\mathbf{x}\right)
-f\left( S_{y^{\prime }}^{k}\mathbf{x}\right) \leq \left\Vert y-y^{\prime
}\right\Vert $. This implies that $f\left( \mathbf{x}\right) {-}\mathbb{E}%
\left[ f\left( S_{X_{k}}^{k}\mathbf{x}\right) \right] \leq 2c$ and also%
\begin{equation*}
V=\frac{1}{2}\sup_{\mathbf{x}\in \mathcal{X}^{m}}\sum_{k=1}^{m}\mathbb{E}%
\left[ \left( f\left( S_{X_{k}}^{k}\mathbf{x}\right) {-}f\left(
S_{X_{k}^{\prime }}^{k}\mathbf{x}\right) \right) ^{2}\right] \leq \frac{1}{2}%
\sum_{k=1}^{m}\mathbb{E}\left[ \left\Vert X_{k}{-}X_{k}^{\prime }\right\Vert
^{2}\right] =\sum_{i=1}^{m}\mathbb{E}\left\Vert X_{i}\right\Vert ^{2}.
\end{equation*}%
By Bernstein's inequality, Theorem \ref{thm:Bernstein-inequality}, for $%
t>0$,%
\begin{equation*}
\Pr \left\{ \left\Vert \sum_{i}X_{i}\right\Vert {-}\mathbb{E}\left[ \left\Vert
\sum_{i}X_{i}\right\Vert \right] >t\right\} \leq \exp \left( \frac{-t^{2}}{%
2\sum_{i}\mathbb{E}\left\Vert X_{i}\right\Vert ^{2}{+}4ct/3}\right) .
\end{equation*}%
Solving for $t$ gives for $\delta >0$%
\begin{equation*}
\Pr \left\{ \left\Vert \sum_{i}X_{i}\right\Vert \leq \mathbb{E}\left[
\left\Vert \sum_{i}X_{i}\right\Vert \right] {+}\sqrt{2\sum_{i}\mathbb{E}%
\left\Vert X_{i}\right\Vert ^{2}\ln \left( 1/\delta \right) }{+}\frac{4c}{3}%
\ln \left( 1/\delta \right) \right\} \leq \delta \text{.}
\end{equation*}%
Using Jensen's inequality, independence, and the mean-zero assumption to
bound $\mathbb{E}\left[ \left\Vert \sum_{i}X_{i}\right\Vert \right] \leq 
\sqrt{\sum_{i}\mathbb{E}\left\Vert X_{i}\right\Vert ^{2}}$ completes the
proof.
\end{proof}

\subsection{Empirical bounds}\label{app:emp_bounds}

If we drop the mean-zero assumption, then the inequality in Theorem \ref{thm:vector-concentration} reads%
\begin{equation*}
\left\Vert \sum_{i}\left( X_{i}-\mathbb{E}\left[ X_{i}\right] \right)
\right\Vert \leq \sqrt{\sum_{i}\mathbb{E}\left\Vert X_{i}-\mathbb{E}\left[
X_{i}\right] \right\Vert ^{2}}\left( 1{+}\sqrt{2\ln \left( 1/\delta \right) }%
\right) {+}\frac{4c}{3}\ln \left( 1/\delta \right) .
\end{equation*}%
In many applications, we observe the $X_{i}$, but we do not know the
expectations and wish to use the inequality to estimate $\sum_{i}\mathbb{E}%
\left[ X_{i}\right] $. In this case, it is useful to have an empirical
estimate of the variance term on the right-hand side. We give two such
estimates, a simple biased one, and an unbiased one for i.i.d data. Both use
the following concentration inequality, which can be found in \citep%
[Corollary 10]{maurer2018empirical}.

\begin{theorem}
\label{thm:selfbound} 
Define an operator $D^{2}$ acting on measurable
functions $f:\mathcal{X}^{m}\rightarrow 
\mathbb{R}
$ by 
\begin{equation*}
D^{2}f\left( \mathbf{x}\right) =\sum_{k}\left( f\left( \mathbf{x}\right)
-\inf_{y\in \mathcal{X}}f\left( S_{y}^{k}\mathbf{x}\right) \right) ^{2},
\end{equation*}%
and let $\mathbf{X}=\left( X_{1},...,X_{m}\right) $ be a vector of
independent variables with values in $\mathcal{X}$. Now suppose $f:\mathcal{X%
}^{m}\rightarrow 
\mathbb{R}
$ satisfies $f\left( \mathbf{x}\right) -\inf_{y\in \mathcal{X}}f\left(
S_{y}^{k}\mathbf{x}\right) \leq b$ for all $k\in \left\{ 1,...,m\right\} $
and all $\mathbf{x}\in \mathcal{X}^{m}$, and for some $a>0$%
\begin{equation}
D^{2}f\left( \mathbf{x}\right) \leq af\left( \mathbf{x}\right) ,\forall 
\mathbf{x}\in \mathcal{X}^{m}\text{.}  \label{selfbound condition}
\end{equation}%
Then for all $\delta >0$ with probability at least $1-\delta $%
\begin{equation*}
\sqrt{\mathbb{E}\left[ f\left( \mathbf{X}\right) \right] }\leq \sqrt{f\left( 
\mathbf{X}\right) }{+}\sqrt{2\max \left\{ a,b\right\} \ln \left( 1/\delta
\right) }.
\end{equation*}
\end{theorem}
\bigskip
For the simple biased estimate, we use the elementary inequality 
\begin{equation}
\sqrt{\sum_{i}\mathbb{E}\left\Vert X_{i}-\mathbb{E}\left[ X_{i}\right]
\right\Vert ^{2}}=\sqrt{\sum_{i}\left( \mathbb{E}\left\Vert X_{i}\right\Vert
^{2}-\left\Vert \mathbb{E}\left[ X_{i}\right] \right\Vert ^{2}\right) }\leq 
\sqrt{\sum_{i}\mathbb{E}\left\Vert X_{i}\right\Vert ^{2}}.
\label{elementary variance estimate}
\end{equation}%
Let $f\left( \mathbf{x}\right) =\sum_{i}\left\Vert x_{i}\right\Vert ^{2}$.
The minimizer in $y$ of $f\left( S_{y}^{k}\mathbf{x}\right) $ is clearly $%
y=0 $. Thus $f\left( \mathbf{x}\right) -\inf_{y\in \mathcal{X}}f\left(
S_{y}^{k}\mathbf{x}\right) \leq \left\Vert x_{k}\right\Vert ^{2}\leq c^{2}$
and 
\begin{equation*}
D^{2}f\left( \mathbf{x}\right) \leq \sum_{k}\left\Vert x_{k}\right\Vert
^{4}\leq c^{2}\sum_{k}\left\Vert x_{k}\right\Vert ^{2}=c^{2}f\left( \mathbf{x%
}\right) .
\end{equation*}%
The theorem above, with $a=b=c^{2}$ gives with probability at least $%
1-\delta $ that%
\begin{equation*}
\sqrt{\sum_{i}\mathbb{E}\left\Vert X_{i}\right\Vert ^{2}}\leq \sqrt{%
\sum_{i}\left\Vert X_{i}\right\Vert ^{2}}{+}c\sqrt{2\ln \left( 1/\delta
\right) }.
\end{equation*}%
A union bound with Theorem \ref{thm:vector-concentration} gives

\begin{proposition}
\label{prop:empirical-biased}Under the conditions of Theorem \ref%
{thm:vector-concentration}, but with uncentered $X_{i}$, we have for $%
0<\delta <1$ with probability at least $1-\delta $ in $\mathbf{X}$%
\begin{equation*}
\left\Vert \sum_{i}\left( X_{i}-\mathbb{E}\left[ X_{i}\right] \right)
\right\Vert \leq \sqrt{\sum_{i}\left\Vert X_{i}\right\Vert ^{2}}\left( 1{+}%
\sqrt{2\ln \left( 2/\delta \right) }\right) {+}\frac{16c}{3}\ln \left(
2/\delta \right).
\end{equation*}
\end{proposition}

For identically distributed $X_{i}$ we can do better, since then%
\begin{equation}
\sqrt{\sum_{i}\mathbb{E}\left\Vert X_{i}-\mathbb{E}\left[ X_{i}\right]
\right\Vert ^{2}}=\sqrt{\frac{1}{2\left( m-1\right) }\mathbb{E}\sum_{i,j:i\neq
j}\left\Vert X_{i}-X_{j}\right\Vert ^{2}},
\label{unbiased variance estimator}
\end{equation}%
and we can apply Theorem \ref{thm:selfbound} to 
\begin{equation*}
f\left( \mathbf{x}\right) =\sum_{i,j:i\neq j}\left\Vert x_{i}-x_{j}\right\Vert
^{2}.
\end{equation*}%
Then for any $k\in \left[ m\right] $ and $y\in \mathcal{X}$ we have 
\begin{equation*}
f\left( \mathbf{x}\right) -f\left( S_{y}^{k}\mathbf{x}\right)
=2\sum_{i:i\neq k}\left( \left\Vert x_{i}-x_{k}\right\Vert ^{2}-\left\Vert
x_{i}-y\right\Vert ^{2}\right) .
\end{equation*}%
It is easy to see that the minimizer in $y$ of $f\left( S_{y}^{k}\mathbf{x}%
\right) $ is the average $y_{k}=1/\left( m-1\right) \sum_{j:j\neq k}x_{j}$.
A computation gives \begin{equation*}
\sum_{i:i\neq k}\left\Vert x_{i}-y_{k}\right\Vert ^{2}=\sum_{i:i\neq
k}\left\Vert x_{i}-x_{k}\right\Vert ^{2}-\frac{1}{m-1}\left\Vert
\sum_{i:i\neq k}\left( x_{i}-x_{k}\right) \right\Vert ^{2}
\end{equation*}%
Thus%
\begin{equation*}
f\left( \mathbf{x}\right) -\inf_{y}f\left( S_{y}^{k}\mathbf{x}\right) =\frac{%
2}{m-1}\left\Vert \sum_{i:i\neq k}\left( x_{i}-x_{k}\right) \right\Vert
^{2}\leq 8\left( m-1\right) c^{2},
\end{equation*}%
Thus, $b = 8\left( m-1\right) c^{2}$, and%
\begin{eqnarray*}
D^{2}f\left( \mathbf{x}\right)  &=&\frac{4}{\left( m-1\right) ^{2}}%
\sum_{k}\left\Vert \sum_{i:i\neq k}\left( x_{i}-x_{k}\right) \right\Vert ^{4}
\\
&=&\frac{4\left( m-1\right) ^{4}}{\left( m-1\right) ^{2}}\sum_{k}\left(
\left\Vert \frac{1}{m-1}\sum_{i:i\neq k}\left( x_{i}-x_{k}\right)
\right\Vert ^{2}\right) ^{2} \\
&\leq &\frac{4\left( m-1\right) ^{4}}{\left( m-1\right) ^{2}}\sum_{k}\left( 
\frac{1}{m-1}\sum_{i:i\neq k}\left\Vert x_{i}-x_{k}\right\Vert ^{2}\right)
^{2} \\
&\leq &16c^{2}\left( m-1\right) \sum_{k}\sum_{i:i\neq k}\left\Vert
x_{i}-x_{k}\right\Vert ^{2}.
\end{eqnarray*}%
The first inequality follows from Jensen's inequality, the second is
obtained by bounding $1/\left( m-1\right) \sum_{i:i\neq k}\left\Vert
x_{i}-x_{k}\right\Vert ^{2}\leq 4c^{2}$. Thus $D^{2}f\left( \mathbf{x}%
\right) \leq 16\left( m-1\right) c^{2}f\left( \mathbf{x}\right)$ where $a = 16\left( m-1\right) c^{2}$.

It follows from Theorem \ref{thm:selfbound} that with probability at
least $1-\delta $%
\begin{equation*}
\sqrt{\sum_{i}\mathbb{E}\left\Vert X_{i}-\mathbb{E}\left[ X_{i}\right]
\right\Vert ^{2}}\leq \sqrt{\frac{1}{2\left( m-1\right) }\sum_{i,j:i\neq
j}\left\Vert X_{i}-X_{j}\right\Vert ^{2}}{+}4c\sqrt{\ln \left( 1/\delta
\right) }
\end{equation*}%
and the union bound with Theorem \ref{thm:vector-concentration} gives
(with some rather crude estimates) for $\delta <2/e$ with probability at
least $1-\delta $%
\begin{eqnarray*}
\left\Vert \sum_{i}\left( X_{i}-\mathbb{E}\left[ X_{i}\right] \right)
\right\Vert  &>&\sqrt{\frac{1}{2\left( m-1\right) }\sum_{i,j:i\neq j}\left\Vert
X_{i}-X_{j}\right\Vert ^{2}}\left( 1{+}\sqrt{2\ln \left( 2/\delta \right) }%
\right)  \\
&&{+}(4\sqrt{2}+4+\frac{4}{3}) c\ln \left( 2/\delta \right) .
\end{eqnarray*}%
Using $4\sqrt{2}+4+\frac{4}{3} = 10.99<11 $ we obtain

\begin{proposition}
\label{Proposition empirical unbiased}Let $\mathbf{X}=\left(
X_{1},...,X_{m}\right) $ be a vector of i.i.d. random variables with values in
a Hilbert space, satisfying $\left\Vert X_{1}\right\Vert \leq c$ almost
surely. Then for $0<\delta <2/e$ with probability at least $1-\delta $ in $%
\mathbf{X}$%
\begin{equation*}
\left\Vert \sum_{i}\left( X_{i}-\mathbb{E}\left[ X_{i}\right] \right)
\right\Vert \leq \sqrt{\frac{1}{2\left( m-1\right) }\sum_{i,j:i\neq j}\left\Vert
X_{i}-X_{j}\right\Vert ^{2}}\left( 1{+}\sqrt{2\ln \left( 2/\delta \right) }%
\right) {+}11c\ln \left( 2/\delta \right) .
\end{equation*}
\end{proposition}

\subsection{Mixing and its consequences} \label{app:mixing}

Let $\mathbf{X}:=\left\{ X_{t}\right\} _{t\in \mathbb{N}}$ be a sequence of
random variables with values in some measurable space $\mathcal{X}$. For a
set $I\subseteq $ $\mathbb{N}$ use $\Sigma \left( I\right) $ for the $\sigma 
$-field generated by $\left\{ X_{i}\right\} _{i\in I}$ and $\mu _{I}$ for
the joint distribution of $\left\{ X_{i}\right\} _{i\in I}$. Then the
definition of the mixing coefficients $\beta_{\mu} \left( \tau \right) $ for $\tau
\in \mathbb{N}$ reads \citep{bradley2005basic}%
\begin{equation*}
\beta_{\mu}\left( \tau \right) :=\sup_{j\in \mathbb{N}}\sup_{B\in
\Sigma \left( \left[ 1,j\right] \cup \left[ j{+}\tau ,\infty \right) \right)
}\left\vert \mu _{\left[ 1,j\right] \cup \left[ j{+}\tau ,\infty \right)
}\left( B\right) -\mu _{\left[ 1,j\right] }\times \mu _{\left[ j{+}\tau
,\infty \right) }\left( B\right) \right\vert .
\end{equation*}

In the sequel, we let $T,m,\tau \in \mathbb{N}$ with $T=2m\tau $ the length
of the trajectory observed. We call $\tau $ the "mixing time" with the
intuitive understanding that events separated by more than $\tau $ can be
regarded as independent, with a penalty in probability of order $\beta_{\mu}\left( \tau \right) $.

For $j\in \left[ m\right] $ define the index sets $I_{j}=\left\{ 2\left(
j-1\right) \tau {+}1,...,\left( 2j-1\right) \tau \right\} $ and $I_{j}^{\prime
}=\left\{ \left( 2j-1\right) \tau {+}1,...,2j\tau \right\} $. Then $\left[ T%
\right] =\bigcup_{k=1}^{m}I_{k}\cup I_{k}^{\prime }$. We also set

\begin{lemma}
\label{lem:betamix_nonstationary}For $m\in \mathbb{N}$ and $B\in \Sigma \left( I_{1}\cup
...\cup I_{m}\right) $%
\begin{equation*}
\left\vert \mu _{I_{1}\cup ...\cup I_{m}}\left( B\right) -\mu _{I_{1}}\times
...\times \mu _{I_{m}}\left( B\right) \right\vert \leq \left( m-1\right)
\beta_{\mu}\left( \tau \right) ,
\end{equation*}%
and for $B^{\prime }\in \Sigma \left( I_{1}^{\prime }\cup ...\cup
I_{m}^{\prime }\right) $%
\begin{equation*}
\left\vert \mu _{I_{1}^{\prime }\cup ...\cup I_{m}^{\prime }}\left(
B^{\prime }\right) -\mu _{I_{1}^{\prime }}\times ...\times \mu
_{I_{m}^{\prime }}\left( B^{\prime }\right) \right\vert \leq \left(
m-1\right) \beta_{\mu}\left( \tau \right) ,
\end{equation*}
\end{lemma}

\begin{proof}
By Fubini's Theorem and the definition of the mixing coefficients, we have
for $1\leq k<m$, that%
\begin{equation*}
\left\vert \mu _{I_{1}}\times ...\times \mu _{I_{k}\cup ...\cup I_{m}}\left(
B\right) -\mu _{I_{1}}\times ...\times \mu _{I_{k}}\times \mu _{I_{k{+}1}\cup
...\cup I_{m}}\left( B\right) \right\vert \leq \beta_{\mu}\left(\tau \right).
\end{equation*}%
Then with a telescopic expansion,
\begin{eqnarray*}
&&\left\vert \mu _{I_{1}\cup ...\cup I_{m}}\left( B\right) -\mu
_{I_{1}}\times ...\times \mu _{I_{m}}\left( B\right) \right\vert \\
& = &\left\vert \sum_{k=1}^{m-1} \mu _{I_{1}}\times ...\times \mu
_{I_{k}\cup ...\cup I_{m}}\left( B\right) -\mu _{I_{1}}\times ...\times \mu
_{I_{k}}\times \mu _{I_{k{+}1}\cup ...\cup I_{m}}\left( B\right) \right\vert \\
&\leq &\sum_{k=1}^{m-1}\left\vert \mu _{I_{1}}\times ...\times \mu
_{I_{k}\cup ...\cup I_{m}}\left( B\right) -\mu _{I_{1}}\times ...\times \mu
_{I_{k}}\times \mu _{I_{k{+}1}\cup ...\cup I_{m}}\left( B\right) \right\vert \\
&\leq &\left( m-1\right) \beta_{\mu}\left( \tau \right) .
\end{eqnarray*}%
The argument for $B^{\prime }$ is analogous.\bigskip
\end{proof}

In the sequel we write $\Pr_{I}$ for the probability measure $\mu
_{I_{1}}\times ...\times \mu _{I_{m}}$ on $\Sigma \left( I_{1}\cup ...\cup
I_{m}\right) $ and $\Pr_{I^{\prime }}$ for $\mu _{I_{1}^{\prime }}\times
...\times \mu _{I_{m}^{\prime }}$ on $\Sigma \left( I_{1}^{\prime }\cup
...\cup I_{m}^{\prime }\right) $. The previous lemma can then be stated more
concisely as $\left\vert \Pr_{I}\left( B\right) -\Pr \left( B\right)
\right\vert \leq \left( m-1\right) \beta_{\mu}\left( \tau \right) $
for $B\in \Sigma \left( I_{1}\cup ...\cup I_{m}\right) $ and $\left\vert
\Pr_{I^{\prime }}\left( B\right) -\Pr \left( B\right) \right\vert \leq
\left( m-1\right) \beta_{\mu}\left( \tau \right) $ for $B\in \Sigma
\left( I_{1}^{\prime }\cup ...\cup I_{m}^{\prime }\right) $.

\BlockingLemma*

\begin{proof}
Write $Y_{j}=\sum_{i\in I_{j}}X_{i}$ and $Y_{j}^{\prime }=\sum_{i\in
I_{j}^{\prime }}X_{i}$. Then%
\begin{equation*}
\left\Vert \sum_{i=1}^{n}X_{i}\right\Vert =\left\Vert
\sum_{j=1}^{m}Y_{j}{+}\sum_{j=1}^{m}Y_{j}^{\prime }\right\Vert \leq \left\Vert
\sum_{j=1}^{m}Y_{j}\right\Vert {+}\left\Vert \sum_{j=1}^{m}Y_{j}^{\prime
}\right\Vert ,
\end{equation*}%
Thus%
\begin{align*}
& \Pr \left\{ \left\Vert \sum_{i=1}^{n}X_{i}\right\Vert >F\left( \mathbf{X}%
\right) {+}F^{\prime }\left( \mathbf{X}\right) \right\} \\
& \leq \Pr \left\{ \left\Vert \sum_{j=1}^{m}Y_{j}\right\Vert {+}\left\Vert
\sum_{j=1}^{m}Y_{j}^{\prime }\right\Vert >F\left( \mathbf{X}\right)
{+}F^{\prime }\left( \mathbf{X}\right) \right\} \\
& \leq \Pr \left\{ \left\Vert \sum_{j=1}^{m}Y_{j}\right\Vert >F\left( 
\mathbf{X}\right) \right\} {+}\Pr \left\{ \left\Vert
\sum_{j=1}^{m}Y_{j}^{\prime }\right\Vert >F^{\prime }\left( \mathbf{X}%
\right) \right\} .
\end{align*}%
The conclusion then follows from applying Lemma \ref{lem:betamix_nonstationary} to the
events $B=\left\{ \left\Vert \sum_{j=1}^{m}Y_{j}\right\Vert >F\left( \mathbf{%
X}\right) \right\} \in \Sigma \left( I_{1}\cup ...\cup I_{m}\right) $ and $%
B^{\prime }=\left\{ \left\Vert \sum_{j=1}^{m}Y_{j}^{\prime }\right\Vert
>F^{\prime }\left( \mathbf{X}\right) \right\} \in \Sigma \left(
I_{1}^{\prime }\cup ...\cup I_{m}^{\prime }\right) $. 
\bigskip
\end{proof}

\subsection{Concentration for dependent variables}\label{app:con_ineq_dependent}

Now let $X_{i}$ be a possibly dependent sequence of mean-zero random vectors
in $H$. Write $Y_{j}=\sum_{i\in I_{j}}X_{i}$ and $Y_{j}^{\prime }=\sum_{i\in
I_{j}^{\prime }}X_{i}$. Note the bounds $\left\Vert Y_{k}\right\Vert \leq
\tau c$ and $\left\Vert Y_{k}^{\prime }\right\Vert \leq \tau c$ and that the 
$Y_{k}$ are independent under $\Pr_{I}$, as are the $Y_{k}^{\prime }$ under $%
\Pr_{I^{\prime }}$. Define 
\begin{eqnarray*}
F\left( \mathbf{X}\right)  &=&\sqrt{\sum_{k=1}^{m}\mathbb{E}\left\Vert
Y_{k}\right\Vert ^{2}}\left( 1{+}\sqrt{2\ln \left( 2/\delta \right) }\right) {+}%
\frac{4\tau c}{3}\ln \left( 2/\delta \right)  \\
F^{\prime }\left( \mathbf{X}\right)  &=&\sqrt{\sum_{k=1}^{m}\mathbb{E}%
\left\Vert Y_{k}^{\prime }\right\Vert ^{2}}\left( 1{+}\sqrt{2\ln \left(
2/\delta \right) }\right) {+}\frac{4\tau c}{3}\ln \left( 2/\delta \right) .
\end{eqnarray*}%

Lemma \ref{lem:blocking} gives%
\begin{eqnarray*}
\Pr \left\{ \left\Vert \sum_{i=1}^{n}X_{i}\right\Vert >F\left( \mathbf{X}%
\right) {+}F^{\prime }\left( \mathbf{X}\right) \right\}  &\leq &\Pr_{I}\left\{
\left\Vert \sum_{j=1}^{m}Y_{j}\right\Vert >F\left( \mathbf{X}\right)
\right\} {+}\Pr_{I^{\prime }}\left\{ \left\Vert \sum_{j=1}^{m}Y_{j}^{\prime
}\right\Vert >F^{\prime }\left( \mathbf{X}\right) \right\} {+}2\left(
m-1\right) \beta_{\mu}\left( \tau \right)  \\
&\leq &\delta /2{+}\delta /2{+}2\left( m-1\right) \beta_{\mu}\left( \tau
\right),
\end{eqnarray*}%
where the second inequality follows from Theorem \ref{thm:vector-concentration}. Substitution of the expressions for $F\left( \mathbf{X}%
\right) $ and $F^{\prime }\left( \mathbf{X}\right) $ and using $\sqrt{a}{+}%
\sqrt{b}\leq \sqrt{2}\sqrt{a{+}b}$, we obtain%
\begin{eqnarray*}
&&\Pr \left\{ \left\Vert \sum_{i=1}^{n}X_{i}\right\Vert >\sqrt{%
\sum_{k=1}^{m}\left( \mathbb{E}\left\Vert Y_{k}\right\Vert ^{2}{+}\mathbb{E}%
\left\Vert Y_{k}^{\prime }\right\Vert ^{2}\right) }\left( \sqrt{2}{+}2\sqrt{%
\ln \left( 2/\delta \right) }\right) {+}\frac{8\tau c}{3}\ln \left( 2/\delta
\right) \right\}  \\
&<&\delta {+}2\left( m-1\right) \beta_{\mu}\left( \tau \right) \text{.}
\end{eqnarray*}%
Now $\mathbb{E}\left\Vert Y_{k}\right\Vert ^{2}=\sum_{\left( t,s\right) \in
I_{k}\times I_{k}}\mathbb{E}\left[ \left\langle X_{t},X_{s}\right\rangle %
\right] $ and $\mathbb{E}\left\Vert Y_{k}^{\prime }\right\Vert
^{2}=\sum_{\left( t,s\right) \in I_{k}^{\prime }\times I_{k}^{\prime }}%
\mathbb{E}\left[ \left\langle X_{t},X_{s}\right\rangle \right] $. Thus 
\begin{equation*}
\sum_{k=1}^{m}\left( \mathbb{E}\left\Vert Y_{k}\right\Vert ^{2}{+}\mathbb{E}%
\left\Vert Y_{k}^{\prime }\right\Vert ^{2}\right) =\sum_{\left( t,s\right)
\in S_{\tau}}\mathbb{E}\left[ \left\langle X_{t},X_{s}\right\rangle \right] ,
\end{equation*}%
where $S_\tau$ is the set of index pairs%
\begin{equation}
S_\tau=\bigcup_{k=1}^{m}\left( I_{k}\times I_{k}\right) \cup \left( I_{k}^{\prime
}\times I_{k}^{\prime }\right) ,  \label{Define S}
\end{equation}%
which is a disjoint union of $\tau \times \tau $-squares along the diagonal
of the $T\times T$ matrix. If $\left( t,s\right) \in S_\tau$ then the times $s$
and $t$ are never more than $\tau $ apart. We have proved

\begin{theorem}
\label{Theorem Dependent Bernstein}Let $m,\tau \in \mathbb{N}$, $T=2m\tau $,
and let $\mathbf{X}=\left( X_{1},...,X_{T}\right) $ be a vector of mean zero
random variables in a separable Hilbert-space $H$, satisfying $\left\Vert
X_{t}\right\Vert \leq c$. Let $\delta >0$. Then with a probability of at least $%
1-\delta -2\left( m-1\right) \beta_{\mu}\left( \tau \right) $ we have%
\begin{equation*}
\left\Vert \sum_{t=1}^{T}X_{t}\right\Vert \leq \sqrt{\sum_{\left( t,s\right)
\in S_{\tau}}\mathbb{E}\left[ \left\langle X_{t},X_{s}\right\rangle \right] }\left( 
\sqrt{2}{+}2\sqrt{\ln \left( 2/\delta \right) }\right) {+}\frac{8\tau c}{3}\ln
\left( 2/\delta \right) ,
\end{equation*}%
where $S_{\tau}\subseteq \left[ T\right] \times \left[ T\right] $ is given by (\ref%
{Define S}).
\end{theorem}

\begin{enumerate}
\item If we drop the mean-zero assumption, this becomes Theorem \ref{thm:Data-dependent-Bernstein}.

\item To interpret the first term on the right-hand side, note the
similarity to the variance of $\sum X_{t}$, which would be%
\begin{equation*}
\mathbb{E}\left\Vert \sum X_{t}\right\Vert ^{2}=\sum_{\left( t,s\right) \in %
\left[ T\right] \times \left[ T\right] }\mathbb{E}\left[ \left\langle
X_{t},X_{s}\right\rangle \right] .
\end{equation*}

\item If the $X_{t}$ are independent, we can set $\tau =1$ and $\beta _{%
\mathbf{X}}\left( \tau \right) =0$, and we recover Theorem \ref{thm:vector-concentration} up to a constant factor of $\sqrt{2}$ on the first,
and $2$ on the second term.
\end{enumerate}

In the proof above, the functions $F\left( \mathbf{X}\right) $ and $F^{\prime
}\left( \mathbf{X}\right) $ were constants. But let 
\begin{equation*}
F\left( \mathbf{X}\right) =\sqrt{\sum_{k=1}^{m}\left\Vert \sum_{i\in
I_{k}}X_{i}\right\Vert ^{2}}\left( 1{+}\sqrt{2\ln \left( 4/\delta \right) }%
\right) {+}\frac{16c\tau}{3}\ln \left( 4/\delta \right) ,
\end{equation*}%
which is $\Sigma \left( I_{1}\cup ...\cup I_{m}\right) $-measurable, and
replace $I_{k}$ by $I_{k}^{\prime }$ for the analogous definition of $%
F^{\prime }\left( \mathbf{X}\right) $. Then Lemma \ref{lem:blocking} and
Proposition \ref{prop:empirical-biased} give a proof of Theorem \ref%
{thm:empirical-Bernstein-biased}. In the same way, defining 
\begin{equation*}
F\left( \mathbf{X}\right) =\sqrt{\frac{1}{2\left( m-1\right) }\sum_{k,l:k\neq
l}\left\Vert \sum_{i\in I_{k}}X_{i}-\sum_{j\in I_{l}}X_{j}\right\Vert ^{2}}%
\left( 1{+}\sqrt{2\ln \left( 4/\delta \right) }\right) {+}11c\tau\ln \left( 4/\delta
\right) ,
\end{equation*}%
together with the corresponding $F^{\prime }\left( \mathbf{X}\right) $,
yields the proof of Theorem \ref{thm:empirical-Bernstein-unbiased}.

\section{Theoretical result for learning dynamical systems} \label{app:LDS_theory}
In this section, we begin by briefly reviewing the notations and definitions pertinent to learning dynamical systems. Next, we establish bounds on the deviation of the true risk from the empirical risk, uniformly across a prescribed set of HS operators on $\RKHS$. Finally, we demonstrate how this analysis, originally formulated for Ivanov regularization, can be extended to Tikhonov regularization.
\subsection{Notation and background}
Let us first review some fundamental concepts related to Markov chains and Koopman operators. Consider $\mathbf{X}:= \left\{X_{t} \colon t\in \mathbb{N} \right\}$, a collection of random variables taking values in a measurable space $(\mathcal{X}, \Sigma_{\mathcal{X}})$, known as the state space. We define $\mathbf{X}$ as a {\em Markov chain} if $\ensuremath{\mathbb P}\{ X_{t+1} \in B \vert X_{[t]} \} = \ensuremath{\mathbb P}\{X_{t + 1} \in B \vert X_t \}$. Additionally, $\mathbf{X}$ is termed {\em time-homogeneous} if there exists a function $p\colon \mathcal{X} \times \Sigma_{\mathcal{X}} \to [0,1]$, known as the {\it transition kernel}, such that for all $(x, B) \in \mathcal{X} \times \Sigma_{\mathcal{X}}$ and each $t \in \mathbb{N}$,
\[
\mathbb{P}\left\{X_{t + 1} \in B \middle| X_{t} = x \right\} = p(x,B).
\]

A broad class of Markov chains includes those with an \textit{invariant measure} $\pi$, which satisfies $\pi(B) = \int_{\mathcal{X}} \pi(dx)p(x,B)$ for $B\in\Sigma_{\mathcal{X}}$. For a {\em time-homogeneous} Markov chain with an invariant (stationary) distribution $\im$ the (stochastic) \textit{Koopman operator} $\Koop \colon \Lii\to\Lii$  is given by \\
\begin{equation}\label{eq:Koopman_app}
A_{\pi} f(x) := \int_{\mathcal{X}} p(x, dy)f(y) = \mathbb{E}\left[f(X_{t {+} 1}) \middle | X_{t} = x\right], \quad f \in L^2_{\pi}(\mathcal{X}), x \in \mathcal{X}.
\end{equation}

In many practical scenarios, the Koopman operator $\Koop$ is unknown, but data from one or more trajectories are available. The framework for operator regression learning introduced in~\citep{Kostic2022} estimates the Koopman operator on $\Lii$ within an RKHS, using an associated feature map $\fH: \X \to \RKHS$. In this vector-valued regression, the population risk functional is defined as
\begin{equation}    
\Risk(G) = \EE_{X\sim\pi,X^+\sim p(\cdot\vert X))}[\norm{\phi(X^+) - G^{}\phi(X)}^{2}_\RKHS],
\end{equation}

and the goal is to learn $\Koop$ by minimizing the risk over a class of operators $G\colon \RKHS \to \RKHS$, using a dataset of consecutive states $\Data := (x_i, x^{+}_{i})_{i=1}^{n}$. A typical setting involves obtaining these states from a single trajectory of the process after it reaches the equilibrium distribution, where $X_0 \sim \pi$ and $X^+_i \equiv X_{i+1} \sim p(\cdot,\vert,X_i)$ for $i = 2,\ldots,n$. A common estimator in this context is the Reduced Rank Regression (RRR) estimator $\EKRR$, obtained by minimizing the regularized empirical risk
\begin{equation}    
{\ERisk}_\lambda(\Estim) := \tfrac{1}{n}\sum_{i \in[n]}\norm{\phi(x^+_i) - \Estim^{}\phi(x_i)}^{2}_{\RKHS} {+} \lambda\|\Estim\|^2_{\text{HS}},
\end{equation}

over operators $\Estim$ of rank at most $r$. The estimator $\ERRR = \ECreg^{-1/2}\SVDr{\ECreg^{-1/2}\ECxy}$ is computed via an $r$-truncated SVD $\SVDr{\cdot}$, where the empirical input and cross covariances are respectively
$$\ECx \,= \,\textstyle{\tfrac{1}{n}\sum_{i \in[n]}}\, \phi(x_{i}){\otimes} \phi(x_{i}), ~~{\rm and}~~\ECxy\,{ =} \,\textstyle{\tfrac{1}{n}\sum_{i \in[n]}}\, \phi(x_{i}){\otimes }\phi(x^+_{i}).$$
In this section we denote as $\ECx_{\reg} = \ECx + \reg I_{\RKHS}$. Further, similar to the relationship between population and empirical risk, the population covariance and cross-covariance are given respectively by
$$ 
\Cx =  \EE_{X \sim \im}[\phi(X) \otimes \phi(X)], ~~{\rm and}~~ \Cxy = \EE_{X\sim\pi,X^+\sim p(\cdot\vert X))}[\phi(X) \otimes \phi(X^+)].$$

\subsection{Uniform bound for Ivanov regularization}

In view of the fact that both Tikhonov regularization and Ivanov regularization versions of the problem are equivalent, we will first focus on the Ivanov regularization formulation, which is more convenient to theoretical analysis\citep{oneto2016tikhonov, luise2019leveraging}. We first state the uniform bound for the Ivanov regularization case. Then, reformulate to Tikhonov regularization since it is more convenient from a computational standpoint. 

\begin{theorem} \label{thm:ivanov_risk_bound}
Let $\mathbf{X}=(X_{t})_{t=1}^{n}$ be a stationary Markov chain with distribution $\mu$, and the risk definitions as above, noting that $\pi=\mu_{1}$.
Denote $\mathcal{G}_{r,\gamma} = \{\Estim \in \mathrm{HS}_r(\mathcal{H}): \|\Estim\|_{\mathrm{HS}} \leq \gamma \}$ and  $\mathbf{Y}=\left( \fH(X_{0})\otimes\fH(X_{0}),...,\fH(X_{n-1})\otimes\fH(X_{n-1})\right) $, $\mathbf{Z}=\left( \fH(X_{0})\otimes\fH(X_{1}),...,\fH(X_{n-1})\otimes\fH(X_{n})\right) $, and $\mathbf{W}=\left( \|\fH(X_{1})\|^2,...,\|\fH(X_{n})\|^2\right) $. Assume $n = 2m\tau$, and exists $\bcon\,>\,0$  such that $\norm{\fH(X_t)}^2{\leq }\bcon$ a.s. for all $t$. Let $\delta > 0$ and assume $\delt = \delta{-}2(\frac{n}{2\tau}{-}1)\beta_{\mu}(\tau) > 0$ and $\deltprime = \delta{-}2(\frac{n}{2\tau}{-}1)\beta_{\mu}(\tau-1) > 0$. Then, with probability at least $1 {-} \delta$ we have for every $\EEstim \in \mathcal{G}_{r,\gamma}$

\begin{equation}
\begin{aligned}
|\Risk(\EEstim)-\ERisk(\EEstim)| &\leq 
\frac{32\gamma^2\bcon\tau}{3n}\ln{\frac{12}{\delt}} +\frac{64\sqrt{r}\gamma\bcon\tau}{3n}\ln{\frac{12}{\deltprime}}+\frac{7\bcon\tau}{3(\frac{n}{2\tau}-1)}\ln{\frac{12}{\delt}} \\
&+ \sqrt{\frac{2\gamma^4\BECorrtau{Y}\tau}{n}\left( 1
+2\ln\frac{12}{\delt}\right)} + \sqrt{\frac{2r\gamma^2\BECorrtau{Z}\tau}{n}\left( 1{+}2\ln\frac{12}{\deltprime}\right)} \\
&+ \sqrt{\frac{2\overline{V}_{\tau}(\mathbf{W})\tau}{n}\ln{\frac{12}{\delt}}}
\end{aligned}
\end{equation}
where $\overline{V}_{\tau}(\mathbf{W}) = \frac{1}{m(m-1)\tau^2} \sum_{1\leq i< j\leq m} (\overline{W}_{i} - \overline{W}_{j})^2 + (\overline{W'}_{i} - \overline{W'}_{j})^2$, $\overline{W}_{j}=\sum_{i\in I_{j}}W_i$ and $\overline{W'}_{j}=\sum_{i\in I_{j}^{\prime }}W_i$. $\BECorrtau{\mathbf{Y}}$ and $\BECorrtau{\mathbf{Z}}$ were defined before.
\end{theorem}
\begin{proof}
We can restate both risk and the empirical risk as follows:
\begin{equation*}
\begin{aligned}
\Risk(\Estim) & =\EE_{(x,y)\sim \rho} \norm{\fH(y) - \Estim^*\fH(x)}^2 \\ & =
\EE_{(x,y)\sim \rho}[\tr\left(\fH\left(X_{i+1}\right) \otimes \fH\left(X_{i+1}\right)\right)-2\left\langle\fH\left(X_{i+1}\right), \Estim^* \fH\left(x_i\right)\right\rangle_{\mathcal{H}}+\operatorname{tr}\left(\Estim \Estim^* \fH\left(x_i\right) \otimes \fH\left(x_i\right)\right)]\\ & =\tr(\Cy)+\tr\left(\Estim \Estim^* \Cx\right)-2\tr\left(\Estim^* \Cxy\right)
\end{aligned}
\end{equation*}
Similarly, we have $\ERisk(\Estim)=\tr(\ECy)+\tr\left(\Estim \Estim^* \ECx\right)-2\tr\left(\Estim^* \ECxy\right)$.

Let $\mathcal{G}_{r,\gamma} = \{\Estim \in \mathrm{HS}_r(\mathcal{H}): \|\Estim\|_{\mathrm{HS}} \leq \gamma \}$. Then for $\forall \Estim \in \mathcal{G}_{r,\gamma}$ we can easily show:
$$
\begin{aligned}
\Risk(\Estim)-\ERisk(\Estim) & =\tr(\Cy - \ECy)+\tr\left(\Estim \Estim^* (\Cx - \ECx)\right)-2\tr\left(\Estim^* (\Cxy - \ECxy)\right) \\
& \leq \tr(\Cy-\ECy)+\gamma^2\|\Cx-\ECx\|+2 \sqrt{r} \gamma\|\Cxy-\ECxy\|
\end{aligned}
$$
where we have used Hölder's inequality to obtain the last two terms. 

First, we use an empirical Bernstein's inequality for bounded random variables, $W_i = \tr(\fH(X_{i+1})\otimes\fH(X_{i+1})) = \|\fH(X_{i+1})\|^2$ and $|W_i|\leq\bcon$ since

\begin{equation*}
\begin{aligned}
\tr(\Cy - \ECy) & = \tr(\frac{1}{n}\sum_{i=1}^{n}\fH(X_{i+1})\otimes\fH(X_{i+1}) - \EE[\fH(X_{i+1})\otimes\fH(X_{i+1})]) \\
& = \frac{1}{n}\sum_{i=1}^{n}\tr(\fH(X_{i+1})\otimes\fH(X_{i+1}) - \EE[\fH(X_{i+1})\otimes\fH(X_{i+1}))] \\
& = \frac{1}{n} \sum_{i=1}^{n}\tr(\fH(X_{i+1})\otimes\fH(X_{i+1})) - \EE[\tr(\fH(X_{i+1})\otimes\fH(X_{i+1}))].
\end{aligned}
\end{equation*}
Now, Write $\overline{W}_{j}=\sum_{i\in I_{j}}W_i$ and $\overline{W'}_{j}=\sum_{i\in I_{j}^{\prime }}W_i$. Note the bounds $\left\Vert \overline{W}_{k}\right\Vert \leq \tau \bcon$ and $\left\Vert \overline{W}_{k}^{\prime }\right\Vert \leq \tau \bcon$ since we know there exists $\bcon\,>\,0$  such that $\norm{\fH(X_t)}^2{\leq }\bcon$ a.s. for all $t$. In addition, we know that the $\overline{W}_{k}$ are independent under $\Pr_{I}$, as are the $\overline{W}_{k}^{\prime }$ under $%
\Pr_{I^{\prime }}$. 

Define 
\begin{eqnarray*}
F\left( \mathbf{W}\right)  &=& \frac{7m\bcon\tau\ln{\frac{2}{\delta'}}}{3(m-1)} + \sqrt{2mV_m(\overline{W})\ln{\frac{2}{\delta'}}}, V_m(\overline{W}) =  \frac{1}{m(m-1)} \sum_{1\leq i< j\leq m} (\overline{W}_{i} - \overline{W}_{j})^2\\
F^{\prime }\left( \mathbf{W}\right)  &=&\frac{7m\bcon\tau\ln{\frac{2}{\delta'}}}{3(m-1)} + \sqrt{2mV_m(\overline{W'})\ln{\frac{2}{\delta'}}}, V_m(\overline{W'}) =  \frac{1}{m(m-1)} \sum_{1\leq i< j\leq m} (\overline{W'}_{i} - \overline{W'}_{j})^2 .
\end{eqnarray*}%
Lemma \ref{lem:blocking} gives%

\begin{align*}
& \Pr \left\{ \left\Vert \sum_{i=1}^{n}\left( W_i-\mathbb{E}%
\left[ W_i\right] \right) \right\Vert >F\left( \mathbf{W}\right)
+F^{\prime }\left( \mathbf{W}\right) \right\}  \\
& \leq \Pr_{I}\left\{ \left\Vert \sum_{k =
1}^{m}(\overline{W}_{k}-\mathbb{E}\left[ \overline{W}_{k}\right])\right\Vert >F\left( \mathbf{W}\right) \right\} +\Pr_{I^{\prime
}}\left\{ \left\Vert \sum_{k = 1}^{m}(\overline{W}_{k}^{\prime
}-\mathbb{E}\left[ \overline{W}_{k}^{\prime
}\right])\right\Vert
>F^{\prime }\left( \mathbf{W}\right) \right\} +2\left( m-1\right) \beta _{\mathbf{W}}\left( \tau \right) , \\
&\leq \delta /2{+}\delta /2{+}2\left( m-1\right) \beta_{\mu}\left(\tau\right),
\end{align*}%
where the second inequality follows from adapting to arbitrarily bounded random variables of Thm 4 \citep{maurer2009empirical}. Substitution of the expressions for $F\left( \mathbf{W}%
\right) $ and $F^{\prime }\left( \mathbf{W}\right) $ and using  $\sqrt{a}{+}%
\sqrt{b}\leq \sqrt{2}\sqrt{a{+}b}$, and $\delta = 2\delta'$, we obtain%
\begin{eqnarray*}
&&\Pr \left\{ \left\Vert \sum_{i=1}^{n}(W_i-\EE[W_i])\right\Vert > \frac{14m\bcon\tau\ln{\frac{4}{\delta}}}{3(m-1)} + \sqrt{4m(V_m(\overline{W})+V_m(\overline{W'}))\ln{\frac{4}{\delta}}} \right\} \\
&<&\delta+2(m-1)\beta_{\mu}(\tau) 
\end{eqnarray*}%

Let define $\overline{V}_{\tau}(\mathbf{W}) = \frac{1}{\tau^2}(V_m(\overline{W}) + V_m(\overline{W'}))$.
Now, if we divide both sides in the probability by $n$ and substitute $\frac{2\tau}{n} = m$ we have
$$
\tr(\Cy - \ECy) \leq \frac{7\bcon\tau}{3(m-1)}\ln{\frac{4}{\delta - 2(m-1)\beta_{\mu}(\tau)}} + \sqrt{\frac{\overline{V}_{\tau}(\mathbf{W})\ln{\frac{4}{\delta - 2(m-1)\beta_{\mu}(\tau)}}}{m}}.
$$

To bound the second and the third terms, we can use Theorem 2 of this paper with the description in the covariance estimation section (for further details, see propositions in the next section). The result is then followed by a union bound.
\end{proof}

\begin{remark}
Theorem \ref{thm:ivanov_risk_bound} can be used in order to derive an excess risk bound for the well-specified case, following the same reasoning as \citep{Kostic2022, luise2019leveraging}. 
\end{remark}
\subsection{Adaptation of risk bound to Tikhonov regularization}

It is known that there is an equivalence between Ivanov and Tikhonov regularization, in the sense that for each class of estimators satisfying the conditions of Ivanov regularization, there exists a corresponding Tikhonov regularization problem. However, to extend this result, we require additional techniques. The next step involves the following lemma (a restatement of Lemma 15.6 from \cite{anthony1999learning}):

\begin{lemma}\label{lem:adaptation}
Suppose $\operatorname{Pr}$ is a probability distribution and

$$
\left\{E\left(\alpha_{1}, \alpha_{2}, \delta\right): 0<\alpha_{1}, \alpha_{2}, \delta \leq 1\right\}
$$

is a set of events, such that\\
(i) For all $0<\alpha \leq 1$ and $0<\delta \leq 1$,

$$
\operatorname{Pr}\{E(\alpha, \alpha, \delta)\} \leq \delta
$$

(ii) For all $0<\alpha_{1} \leq \alpha \leq \alpha_{2} \leq 1$ and $0<\delta_{1} \leq \delta \leq 1$

$$
E\left(\alpha_{1}, \alpha_{2}, \delta_{1}\right) \subseteq E(\alpha, \alpha, \delta)
$$

Then for $0<a, \delta<1$,

$$
\operatorname{Pr} \bigcup_{\alpha \in(0,1]} E(\alpha a, \alpha, \delta \alpha(1-a)) \leq \delta
$$
\end{lemma}

\TikhonovRiskBound*
\begin{proof}
Based on Theorem 19 in \citep{luise2019leveraging}, we know that $\EEstim_{r,\lambda}$, the minimizer of the rank-reduced Tikhonov regularized problem,
$$
\min_{\Estim \in \mathrm{HS}r(\mathcal{H})} \ERisk^\lambda(\Estim),
$$
with $\lambda$ as a regularization parameter, is also the minimizer of the corresponding Ivanov problem, where $\gamma(\lambda) = \|\EEstim{r,\lambda}\|_{\mathrm{HS}} = \| \ECx_{\lambda}^{-1/2}\SVDr{\ECx_{\lambda}^{-1/2}\ECxy}\|_{\mathrm{HS}}$. However, we cannot simply substitute $\|\EEstim{r,\lambda}\|_{\mathrm{HS}}$ in the formula, as it is a random variable and depends on the data, but using the method outlined in Lemma \ref{lem:adaptation} we can obtain the proof with the choices $\gamma(\lambda) = 2\|\EEstim_{r,\lambda}\|_{\mathrm{HS}}$ and the substitution of $\frac{\delta}{2\|\EEstim{r,\lambda}\|_{\mathrm{HS}}}$ in place of $\delta$.
Indeed, let $Q\left( \gamma ,\delta \right) $ be the right hand side of the
inequality in Theorem \ref{thm:ivanov_risk_bound}, and define the events 
\[
E\left( \alpha _{1},\alpha _{2},\delta \right) :=\left\{ \exists
G:\left\Vert G\right\Vert _{HS}\leq \alpha _{2}^{-1},\left\vert \mathcal{R}%
\left( \hat{G}\right) -\mathcal{\hat{R}}\left( \hat{G}\right) \right\vert
>Q\left( \alpha _{1}^{-1},\delta \right) \right\} .
\]%
Then by Theorem \ref{thm:ivanov_risk_bound} these events satisfy (i) of Lemma \ref{lem:adaptation} and they also
have the monotonicity property (ii).  Using
the lemma with $a=1/2$ then gives the inequality for every $G$ satisfying $\|G\|_{HS}\ge  1$.
 \end{proof}

\begin{remark}
If the process is known to be stationary, we can use the unbiased estimators for both the covariance, $\UECorrtau{Y}$, and the cross-covariance, $\UECorrtau{Z}$, respectively
\end{remark}
\begin{remark}
We can easily adapt the above results for non-stationary processes by simply modifying the definition of risk.
\end{remark}
\section{Experiments} \label{app:exp}
To facilitate the reproducibility of our main experimental results, we have made the code publicly available in an anonymous GitHub repository, which can be accessed via the link \url{https://github.com/erfunmirzaei/EBI4LDS}.
\subsection{Covariance estimation using samples from Ornstein–Uhlenbeck process}\label{app:cov_est}
\subsubsection{Theoretical results}
As mentioned in the section on covariance estimation for estimating the concentration of covariance operators in the Hilbert-Schmidt norm,  the observed vectors are operators, and $X_{t}$
should be replaced by the operator $\fH(X_{t})\otimes \fH(X_{t})$.

Recent studies have relied on Pinelis and Sakhanenko’s inequality, see \citep[][Proposition 2]{caponnetto2007}. As mentioned earlier, the method of blocks and $\beta$-mixing extends this i.i.d. analysis of transfer operator regression to more realistic settings, such as learning from data trajectories of a stationary process. We begin by restating Pinelis and Sakhanenko’s inequality for convenience.

\begin{proposition}\label{prop:con_ineq_ps}
Let $A_i$, $i\in[n]$ be independent random variables with values in a
separable Hilbert space with norm $\norm{\cdot}$. If there exist constants $L>0$ and $\sigma>0$ such that $\forall i \in [n]$, $\|A_i\| \leq \frac{L}{2} $ and $\EE[\|A_i\|^2] \leq \sigma^2$, 
then with probability at least $1-\delta$ 
\begin{equation}\label{eq:con_ineq_ps}
\left\|\frac{1}{n}\sum_{i\in[n]}A_i - \EE[A_i] \right\|\leq 2(\frac{L}{n}{+}\frac{\sigma}{\sqrt{n}})\log(\frac{2}{\delta})
\end{equation}
\end{proposition}

 We now aim to adapt Pinelis and Sakhanenko’s inequality for covariance estimation in data-dependent scenarios.
\begin{proposition}\label{prop:Pinelis_cov}
Let $\mathbf{X}=(X_{t})_{t=1}^{n}$ be a Markov chain with distribution $\mu$ and $\mathbf{Y} = (\fH(X_{t})\otimes\fH(X_{t}))_{t=1}^{n}$. Assume $n =2m\tau$, and exists $\bcon\,>\,0$  such that $\norm{\fH(X_t)}^2{\leq }\bcon$ a.s. for all $t$. Let $\delta \geq 0$ and assume $\delt = \delta{-}2(\frac{n}{2\tau}{-}1)\beta_{\mu}(\tau) > 0$, with probability at least $1 {-} \delta$ we have
\[ 
\norm{ \ECx {-} \Cx }\leq \frac{4\bcon}{m}\ln{\frac{4}{\delt}}{+}\frac{2\bcon}{\sqrt{m}} \ln{\frac{4}{\delt}} ,
\] 

where $\ECx \,= \,\textstyle{\tfrac{1}{n}\sum_{t \in[n]}}\, \phi(x_{t}){\otimes} \phi(x_{t})$ and in general $\Cx \,= \,\textstyle{\tfrac{1}{n}\sum_{t \in[n]}}\, \EE[\phi(X_{t}){\otimes} \phi(X_{t})].$
\end{proposition}

\begin{proof}
As we mentioned before, for covariance estimation, our random variables are the rank-one operators $Y_t = \phi(X_{t})\otimes \phi(X_{t})$ instead of $X_{t}$.
To prove it, we need to use the same procedure as section \ref{app:con_ineq_dependent} with these new random variables and use Pinelis and Sakhanenko's inequality as stated above. 

Now, Write $\overline{Y}_{j}=\sum_{i\in I_{j}}Y_{i}$ and $\overline{Y'}_{j}=\sum_{i\in I_{j}^{\prime }}Y_{i}$. Note the bounds $\left\Vert \overline{Y}_{k}\right\Vert \leq \tau \bcon$ and $\left\Vert \overline{Y'}_{k}\right\Vert \leq \tau \bcon$ since there exists $\bcon\,>\,0$ such that $\norm{\fH(X_t)}^2{\leq }\bcon$ a.s. for all $t$ . In addition, we know that the $\overline{Y}_{k}$ are independent under $\Pr_{I}$, as are the $\overline{Y'}_{k}$ under $%
\Pr_{I^{\prime }}$. Then, we also have the following:
\begin{equation*}
\begin{aligned}
    \|\overline{Y}_k\|^{2}_{HS} = \scalarp{\sum_{i=1}^{\tau} \fH(X_i)\otimes\fH(X_i),  \sum_{j=1}^{\tau}\fH(X_j)\otimes\fH(X_j)}_{HS} \\
    =  \sum_{i=1}^{\tau} \sum_{j=1}^{\tau} \scalarp{\fH(X_i)\otimes\fH(X_i), \fH(X_j)\otimes\fH(X_j)}_{HS} \\
    =   \sum_{i=1}^{\tau} \sum_{j=1}^{\tau} \tr((\fH(X_i)\otimes\fH(X_i))(\fH(X_j)\otimes\fH(X_j))) \\
    =  \sum_{i=1}^{\tau} \sum_{j=1}^{\tau} \scalarp{\fH(X_i), \fH(X_j)}^2  \leq \tau^2 \bcon^{2}
\end{aligned}
\end{equation*}
Therefore, we can set $ L = L' = 2 \tau \bcon$ and $\sigma = \sigma' = \tau \bcon$ for using for $\overline{Y}_{k}$ and $\overline{Y'}_{k}$. 

Define 
\begin{eqnarray*}
F\left( \mathbf{Y}\right) &=& F^{\prime }\left( \mathbf{Y}\right) =(4 \tau \bcon{+}2\tau \bcon\sqrt{m})\log(\frac{2}{\delta'})
\end{eqnarray*}%
Lemma \ref{lem:blocking} gives%

\begin{align*}
& \Pr \left\{ \left\Vert \sum_{i=1}^{n}\left( Y_{i}-\mathbb{E}%
\left[ Y_{i}\right] \right) \right\Vert >F\left( \mathbf{Y}\right)
+F^{\prime }\left( \mathbf{Y}\right) \right\}  \\
& \leq \Pr_{I}\left\{ \left\Vert \sum_{k =
1}^{m}(\overline{Y}_{k}-\mathbb{E}\left[ \overline{Y}_{k}\right])\right\Vert >F\left( \mathbf{Y}\right) \right\} +\Pr_{I^{\prime
}}\left\{ \left\Vert \sum_{k = 1}^{m}(\overline{Y}_{k}^{\prime
}-\mathbb{E}\left[ \overline{Y}_{k}^{\prime
}\right])\right\Vert
>F^{\prime }\left( \mathbf{Y}\right) \right\} +2\left( m-1\right) \beta _{\mu}\left( \tau \right) , \\
&\leq \delta /2{+}\delta /2{+}2\left( m-1\right) \beta_{\mu}\left(\tau\right),
\end{align*}%
where the second inequality follows from Prop. \ref{prop:con_ineq_ps}. Substitution of the expressions for $F\left( \mathbf{Y}%
\right) $ and $F^{\prime }\left( \mathbf{Y}\right) $ and using $\delta = 2\delta'$, we obtain%
\begin{eqnarray*}
&&\Pr \left\{ \left\Vert \sum_{i=1}^{n}(Y_{i}-\EE[Y_{i}]\right\Vert > (8\tau \bcon{+}4\tau \bcon\sqrt{m})\log(\frac{4}{\delta}) \right\} <\delta+2(m-1)\beta_{\mu}(\tau) 
\end{eqnarray*}%

Now, if we divide both sides in the probability by $n$ and substitute $\frac{2\tau}{n} = m$ the proof is finished.
\end{proof}

For comparison, we propose the following proposition, where Theorem \ref{thm:vector-concentration} is used in place of Pinelis and Sakhanenko’s inequality.
\begin{proposition}\label{prop:BI_cov}
Let $\mathbf{X}=(X_{t})_{t=1}^{n}$ be a Markov chain with distribution $\mu$ and $\mathbf{Y} = (\fH(X_{t})\otimes\fH(X_{t}))_{t=1}^{n}$. Assume $n =2m\tau$, and exists $\bcon\,>\,0$  such that $\norm{\fH(X_t)}^2{\leq }\bcon$ a.s. for all $t$. Let $\delta \geq 0$ and assume $\delt = \delta{-}2(\frac{n}{2\tau}{-}1)\beta_{\mu}(\tau) > 0$, with probability at least $1 {-} \delta$ we have
\[  \norm{ \ECx {-} \Cx }\leq \frac{4\bcon}{3m}\ln \frac{2}{\delt} {+}\sqrt{\frac{2\bcon^2}{m}\left(1 {+} 2\ln \frac{2}{\delt} \right)},
\]
where $\ECx \,= \,\textstyle{\tfrac{1}{n}\sum_{t \in[n]}}\, \phi(x_{t}){\otimes} \phi(x_{t})$ and in general $\Cx \,= \,\textstyle{\tfrac{1}{n}\sum_{t \in[n]}}\, \EE[\phi(X_{t}){\otimes} \phi(X_{t})].$
\end{proposition}

\begin{proof}
As we mentioned before, for covariance estimation, our random variables are the rank-one operators $Y_t = \phi(X_{t})\otimes \phi(X_{t})$ instead of $X_{t}$.
To prove we need to use the same procedure as section \ref{app:con_ineq_dependent} with these new random variables and use Theorem \ref{thm:vector-concentration}. 

Now, Write $\overline{Y}_{j}=\sum_{i\in I_{j}}Y_{i}$ and $\overline{Y'}_{j}=\sum_{i\in I_{j}^{\prime }}Y_{i}$. Note the bounds $\left\Vert \overline{Y}_{k}\right\Vert \leq \tau \bcon$ and $\left\Vert \overline{Y'}_{k}\right\Vert \leq \tau \bcon$ since there exists $\bcon\,>\,0$  such that $\norm{\fH(X_t)}^2{\leq }\bcon$ a.s. for all $t$. In addition, we know that the $\overline{Y}_{k}$ are independent under $\Pr_{I}$, as are the $\overline{Y'}_{k}$ under $%
\Pr_{I^{\prime }}$. 

Define 
\begin{eqnarray*}
F\left( \mathbf{Y}\right)  &=&\sqrt{\sum_{k=1}^{m}\mathbb{E}\left\Vert \overline{Y}_{k}-\mathbb{E}\left[
\overline{Y}_{k}\right] \right\Vert ^{2}}\left( 1{+}\sqrt{2\ln \left( 1/\delta \right) }%
\right) {+}\frac{4\bcon\tau}{3}\ln \left( 1/\delta \right)   \\
F^{\prime }\left( \mathbf{Y}\right)  &=&\sqrt{\sum_{k=1}^{m}\mathbb{E}\left\Vert \overline{Y'}_{k}-\mathbb{E}\left[
\overline{Y'}_{k}\right] \right\Vert ^{2}}\left( 1{+}\sqrt{2\ln \left( 1/\delta \right) }%
\right) {+}\frac{4\bcon\tau}{3}\ln \left( 1/\delta \right) .
\end{eqnarray*}%
Lemma \ref{lem:blocking} gives%

\begin{align*}
& \Pr \left\{ \left\Vert \sum_{i=1}^{n}\left( Y_{i}-\mathbb{E}%
\left[ Y_{i}\right] \right) \right\Vert >F\left( \mathbf{Y}\right)
+F^{\prime }\left( \mathbf{Y}\right) \right\}  \\
& \leq \Pr_{I}\left\{ \left\Vert \sum_{k =
1}^{m}(\overline{Y}_{k}-\mathbb{E}\left[ \overline{Y}_{k}\right])\right\Vert >F\left( \mathbf{Y}\right) \right\} +\Pr_{I^{\prime
}}\left\{ \left\Vert \sum_{k = 1}^{m}(\overline{Y}_{k}^{\prime
}-\mathbb{E}\left[ \overline{Y}_{k}^{\prime
}\right])\right\Vert
>F^{\prime }\left( \mathbf{Y}\right) \right\} +2\left( m-1\right) \beta _{\mu}\left( \tau \right) , \\
&\leq \delta /2{+}\delta /2{+}2\left( m-1\right) \beta_{\mu}\left(\tau\right),
\end{align*}%
where the second inequality follows after dropping the mean-zero assumption of Thm. \ref{thm:vector-concentration}. Substitution of the expressions for $F\left( \mathbf{Y}%
\right) $ and $F^{\prime }\left( \mathbf{Y}\right) $ and using  $\sqrt{a}{+}%
\sqrt{b}\leq \sqrt{2}\sqrt{a{+}b}$, and $\delta = 2\delta'$, we obtain%
\begin{eqnarray*}
&&\Pr \left\{ \left\Vert \sum_{i=1}^{n}(Y_{i}-\EE[Y_{i}]\right\Vert > \sqrt{2\sum_{k=1}^{m}(\mathbb{E}\left\Vert \overline{Y}_{k}-\mathbb{E}\left[
\overline{Y}_{k}\right] \right\Vert^{2}{+}\mathbb{E}\left\Vert \overline{Y'}_{k}-\mathbb{E}\left[
\overline{Y'}_{k}\right] \right\Vert^{2})}\left( 1{+}\sqrt{2\ln \left( 2/\delta \right) }%
\right) {+}\frac{8\bcon\tau}{3}\ln \left(2/\delta \right) \right\} \\
&<&\delta+2(m-1)\beta_{\mu}(\tau) 
\end{eqnarray*}%

We know $\sum_{k=1}^{m}(\mathbb{E}\left\Vert \overline{Y}_{k}-\mathbb{E}\left[
\overline{Y}_{k}\right] \right\Vert^{2}{+}\mathbb{E}\left\Vert \overline{Y'}_{k}-\mathbb{E}\left[
\overline{Y'}_{k}\right] \right\Vert^{2}) = 2m\tau^2\bcon^2$.
Now, if we divide both sides in the probability by $n$ and substitute $\frac{2\tau}{n} = m$ and using  $\sqrt{a}{+}%
\sqrt{b}\leq \sqrt{2}\sqrt{a{+}b}$ the proof is finished.    
\end{proof}
\subsubsection{Experimental results}\label{app:cov_est_exp_res}
As stated in the paper, the following experimental results demonstrate the consistent monotonic increase of the covariance upper bound with respect to $\tau$.

\begin{figure}[H]
    \centering
\includegraphics[width=\textwidth]{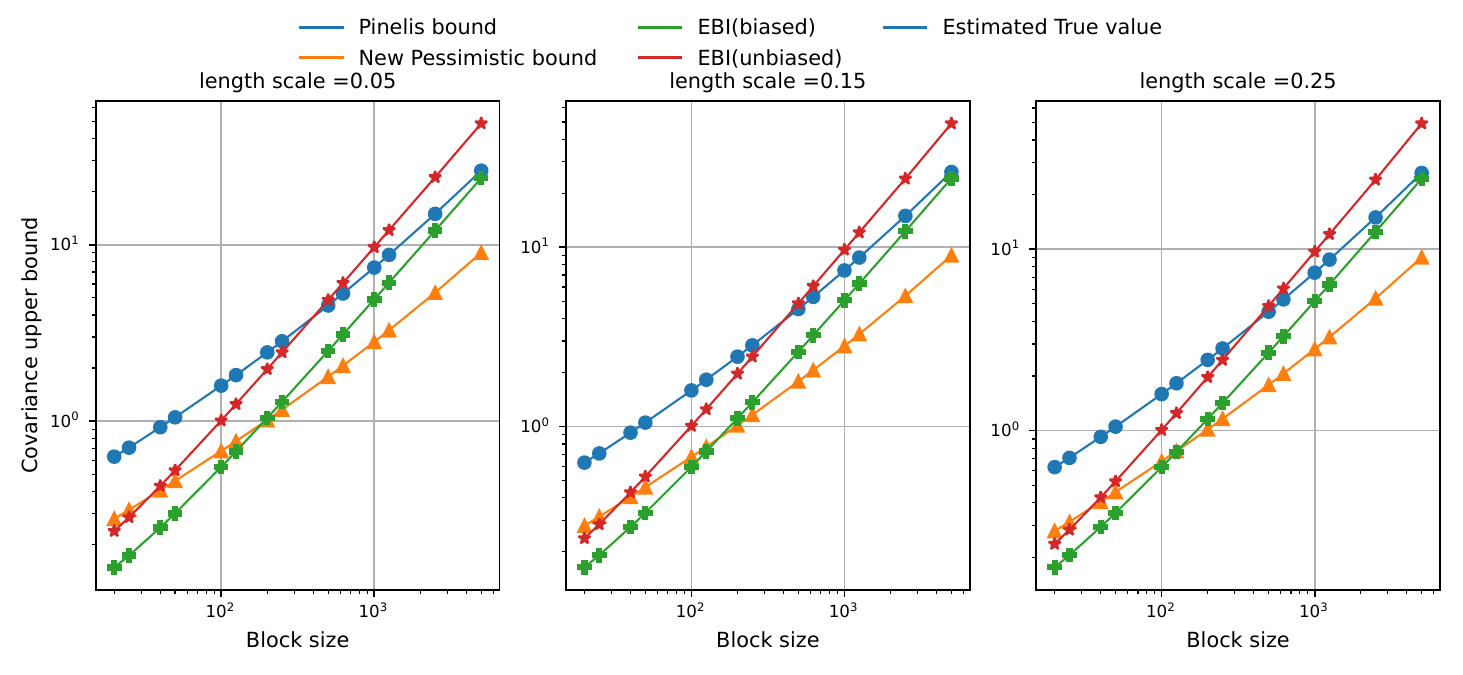}
    \caption{Covariance upper bound as a function of the block size for three different length scales of Gaussian kernel in logarithmic scale. The failure probability is set to 0.05, and we used 10k training points. The plots have been averaged over 30 independent simulations.}
    \label{fig:OU_tau_10k}
\end{figure}

\begin{figure}[H]
    \centering
\includegraphics[width=\textwidth]{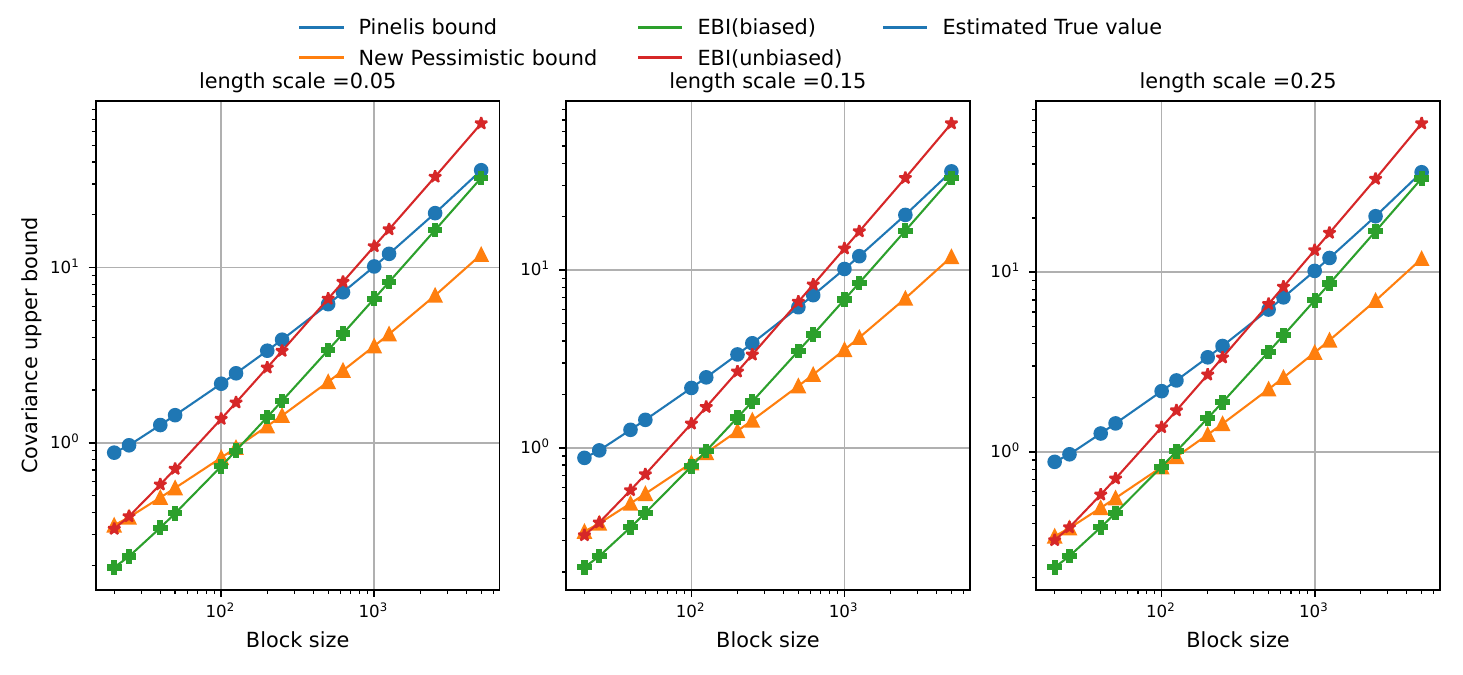}
    \caption{Covariance upper bound as a function of the block size for three different length scales of Gaussian kernel in logarithmic scale. The failure probability is set to 0.01, and we used 10k training points. The plots have been averaged over 30 independent simulations.}
    \label{fig:OU_tau_10k_delta}
\end{figure}

\begin{figure}[H]
    \centering
\includegraphics[width=\textwidth]{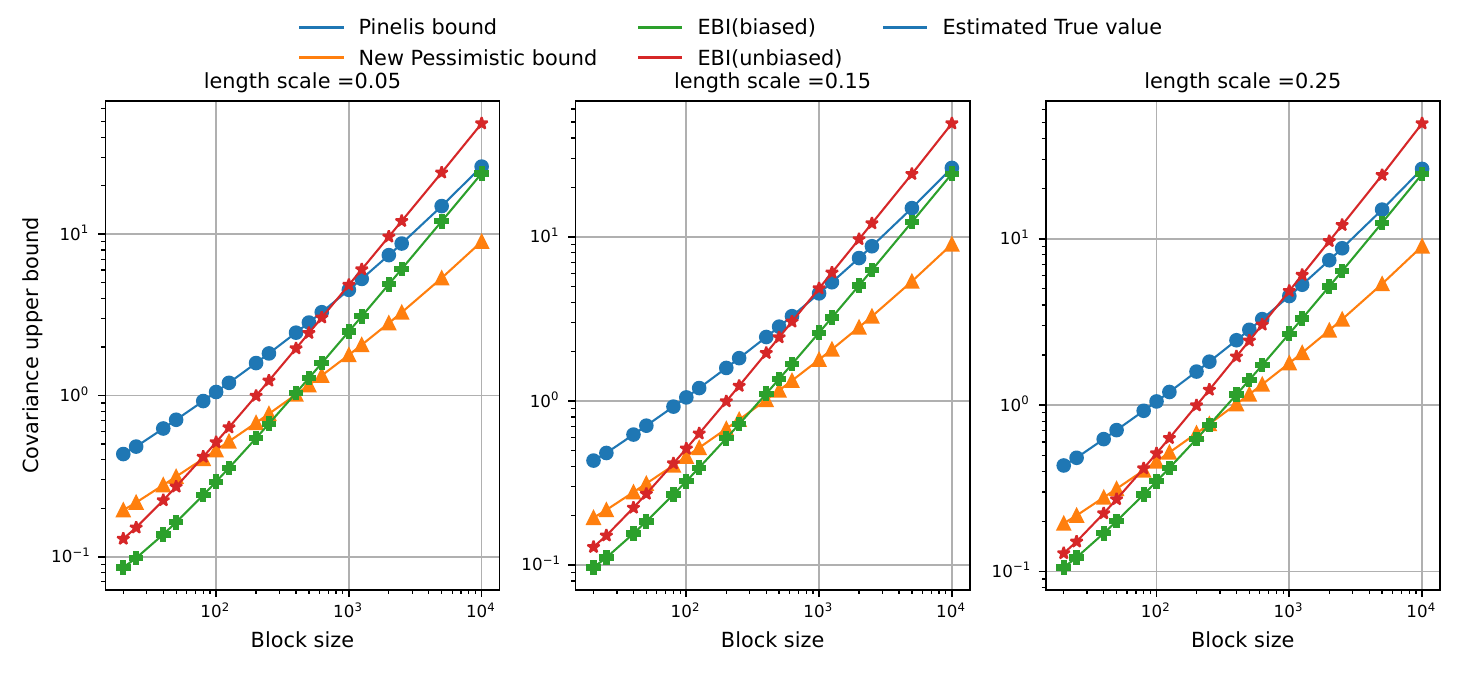}
    \caption{Covariance upper bound as a function of the block size for three different length scales of Gaussian kernel in logarithmic scale. The failure probability is set to 0.05, and we used 20k training points. The plots have been averaged over 30 independent simulations.}
    \label{fig:OU_tau_20k}
\end{figure}

\subsubsection{Estimating the true error of covariance operator}\label{app:cov_est_true_error}
For this experiment, we tried to estimate the norm of the difference between the sample covariance operator and the true one as follows,

\begin{equation*}
    \| \ECx - \Cx \|^{2}_{\HS{\RKHS}} = \tr(\ECx^2) + \tr(\Cx^2) - 2\tr(\ECx\Cx). 
\end{equation*}

Since we have the $\ECx$ operator based on the data, finding its trace is easy. Furthermore, in this example, we know that the invariant distribution is the standard normal Gaussian. Thus, we can find the analytic form for the $\tr(\Cx)$. First, recall $\Cx = \EE_{x \sim \im}[\fH(x) \otimes \fH(x)]$. Then, we have
\begin{equation*}
\begin{aligned}
\tr(\Cx^2) = tr(\EE_{x \sim \im}[\fH(x) \otimes \fH(x)]\EE_{y \sim \im}[\fH(y) \otimes \fH(y)]) \\ = \EE_{x \sim \mathcal{N}(0,1)}\EE_{y \sim \mathcal{N}(0,1)}[\scalarp{\fH(x), \fH(y)}^2] 
 = \EE_{x \sim \mathcal{N}(0,1)}\EE_{y \sim \mathcal{N}(0,1)}[k(x,y)^2]
\\ = \frac{1}{2\pi}\int_{x}\int_{y} (\exp{(\frac{-(x-y)^2}{2l^2})})^{2} \exp{(\frac{-y^2}{2})}\exp{(\frac{-x^2}{2})}dy dx  = \sqrt{\frac{1}{1+\frac{4}{l^2}}},
\end{aligned}
\end{equation*}
where $l$ is the length scale of the Gaussian kernel. Finally, for the last term, we tried to estimate it by $\tr(\ECx, \Tilde{\Cx})$, where $\Tilde\Cx = \frac{1}{n} \sum_{i=1}^{n} \fH(\Tilde{x}_i) \otimes \fH(\Tilde{x}_i)$ and $\Tilde{x}_{i}$s are new samples from OU process. Then we repeat the previous process many times, which means $\tr(\ECx \Cx) \approx \frac{1}{T}\sum_{t=1}^{T}k(\Tilde{X_t}, X)^2$, where $\Tilde{X_t}$ is a new batch of data with the size n at each time. We set $T = 100$ and $n = 10^4$ for this experiment.

\subsubsection{Extreme examples}\label{app:cov_est_ext_exp}
We introduced two empirical Bernstein's inequalities in which two different estimations of the variance proxy have been used. In this section, we will take a more precise look at these estimations for extreme examples in the case of covariance estimation. 
\begin{itemize}
    \item[1)] $K^{.2} = c^2 \Id_{n\times n}$, This means that $K(x,y) = 0$ unless $x=y$. Thus, if we have a very small length scale for the Gaussian Kernel(close to zero) this can happen. In this situation, $\UECorrtau{X} = \BECorrtau{X} = \frac{c^2}{\tau}$.
    \item[2)] $K^{.2} = \Diag(c^2 1_{\tau\times \tau})$, This means that $K(x,y) = 0$ unless $x$ and $y$ are less than $\tau$-time separated and the boundary is very sharp. In other words, if the distance of $x$ and $y$ is less than $\tau$ time steps then $k(x,y) = c^2$ otherwise $0$. In this situation, $\UECorrtau{X} = \BECorrtau{X} = c^2$.
    
    \item[3)] $K^{.2} = c^2 1_{n\times n}$, This means that $K(x,y) = c^2$ no matter what are x and y. Thus, if we have a very large length scale for the Gaussian kernel (close to infinity) this would be the case. In this situation, $\UECorrtau{X} = c^2$ and $\BECorrtau{X} = 0$.
\end{itemize}

Thus, we can conclude that $\frac{b^2}{\tau} \leq \UECorrtau{X} \leq c^2$ and $0 \leq \BECorrtau{X} \leq c^2$ where $b$ and $c$ could be the infimum and supremum values of the kernel, respectively.
\subsection{Noisy ordered MNIST} \label{app:MNIST}

The architecture of the oracle network used for DPNet
is given by $\phi_{\boldsymbol{\theta}}$ : Conv2d $(1,16 ; 5) \rightarrow \operatorname{ReLU} \rightarrow \operatorname{MaxPool}(2) \rightarrow \operatorname{Conv} 2 \mathrm{~d}(16,32 ; 5) \rightarrow \operatorname{ReLU} \rightarrow \operatorname{MaxPool}(2) \rightarrow \operatorname{Dense} (1568,5)$. We set the seed number 42 and used 2 as a padding hyperparameter. Here, the arguments of the convolutional layers are Conv2d( \#in channel \#out channels; kernel size). The Tikhonov regularization parameter for the CNN kernel $\gamma_{\mathrm{CNN}}=10^{-4}$. The network $\phi_{\boldsymbol{\theta}}$ has been pre-trained as a digit classifier using the cross entropy loss function. The training was performed with the Adam optimizer (learning rate $=0.01$ ) for 20 epochs (batch size $=64$ ). The training dataset corresponds to the same 1000 images used to train the Koopman estimators.

In Figure \ref{fig:MNIST_TSNE}, t-SNE, a nonlinear dimension reduction, was applied to the concatenation of left and right eigenfunctions on the test dataset. The results indicate that models with superior representations tend to perform better and exhibit better generalization.
\begin{figure}[H]
    \centering
\includegraphics[width=\textwidth]{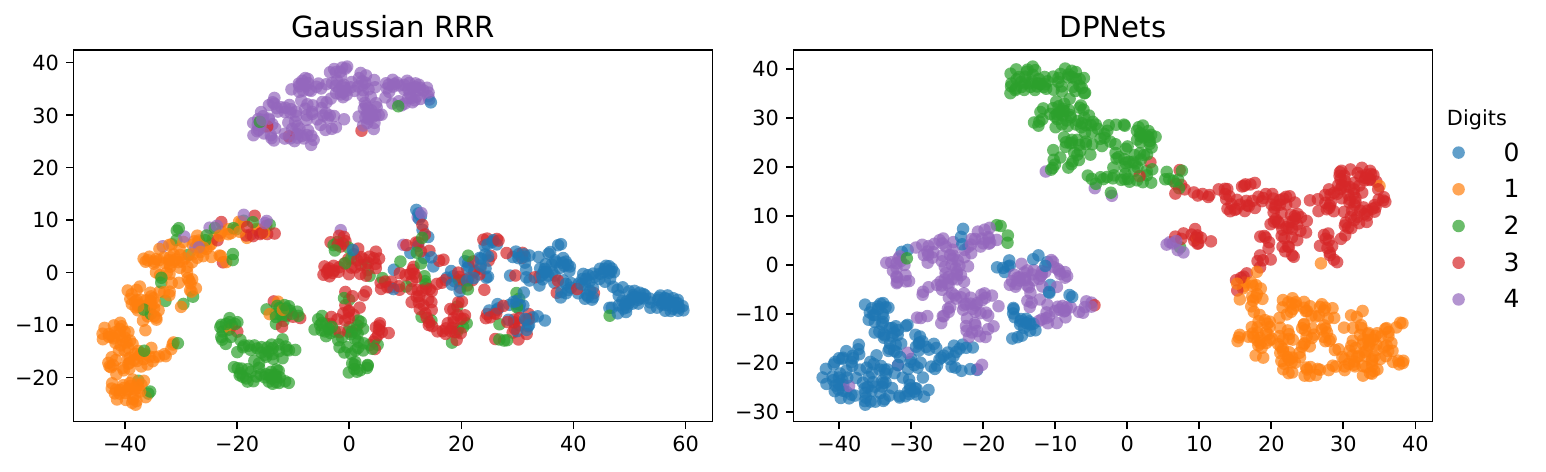}
    \vspace{-.3truecm}
    \caption{"t-SNE visualization of the concatenated left and right eigenfunctions on the MNIST test dataset with $\eta = 0.1$.}
    \label{fig:MNIST_TSNE}
\end{figure}
Here are the results for two different choices of $\eta = 0.2$ and $\eta = 0.05$. 
\begin{figure}[H]
    \centering
    \begin{subfigure}[b]{0.9\textwidth} 
        \centering
        \includegraphics[width=\linewidth]{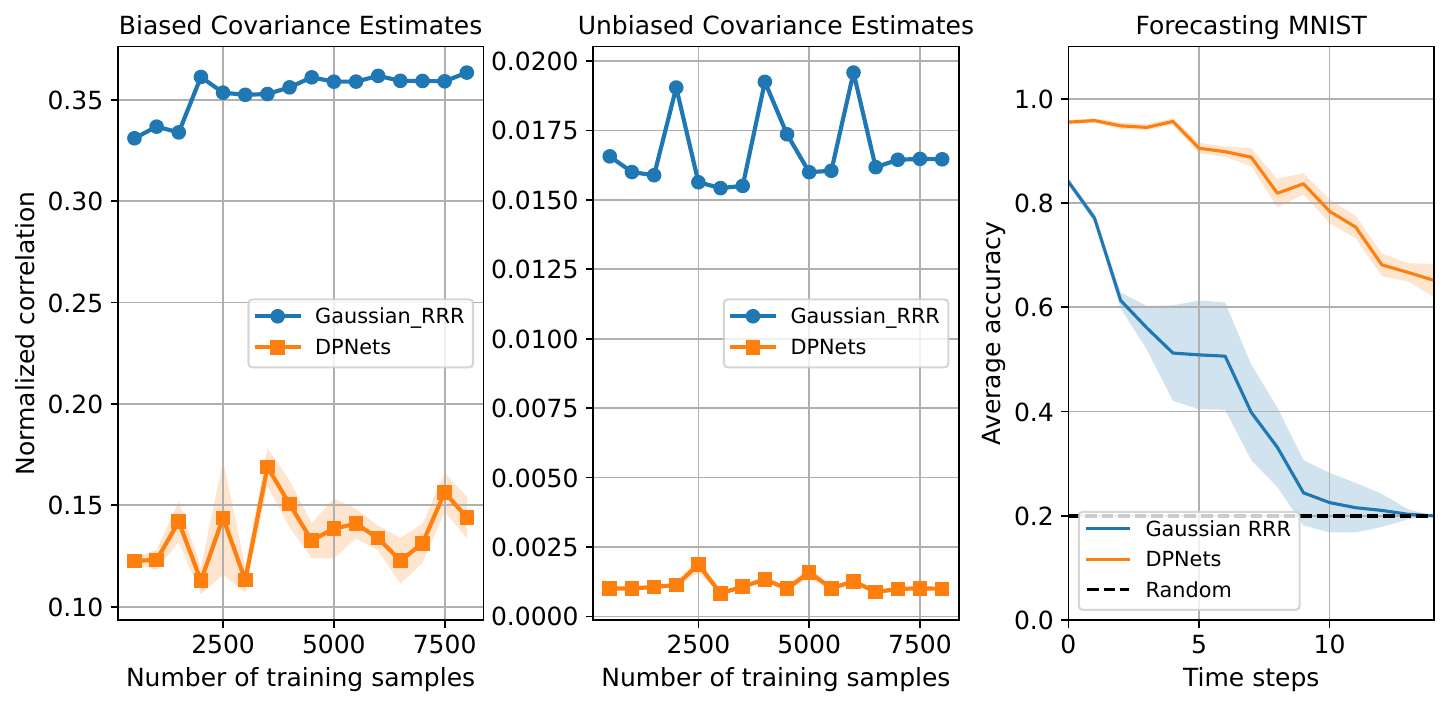}
    \end{subfigure}
    \vspace{0.2cm} 
    \begin{subfigure}[b]{0.9\textwidth} 
        \centering
        \includegraphics[width=\linewidth]{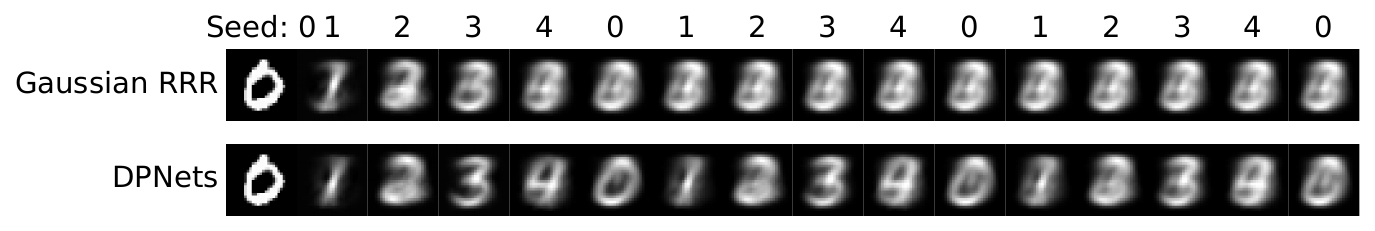}
    \end{subfigure}
    \caption{Performance evaluation of rank-5 RRR estimators using Gaussian and DPNet kernels on MNIST with $\eta = 0.2$: normalized correlations and forecast accuracy. }
    \label{fig:main}
\end{figure}

\begin{figure}[H]
    \centering
\includegraphics[width=\textwidth]{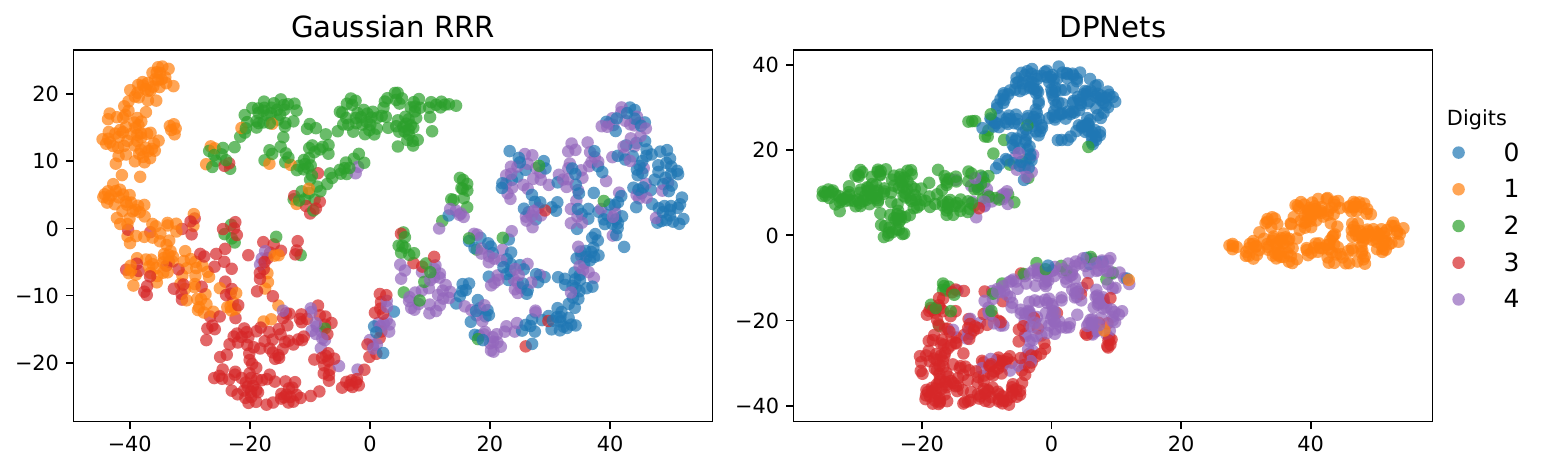}
    \caption{"t-SNE visualization of the concatenated left and right eigenfunctions on the MNIST test dataset with $\eta = 0.2$}
\end{figure}

\begin{figure}
    \centering
    \begin{subfigure}[b]{0.9\textwidth} 
        \centering
        \includegraphics[width=\linewidth]{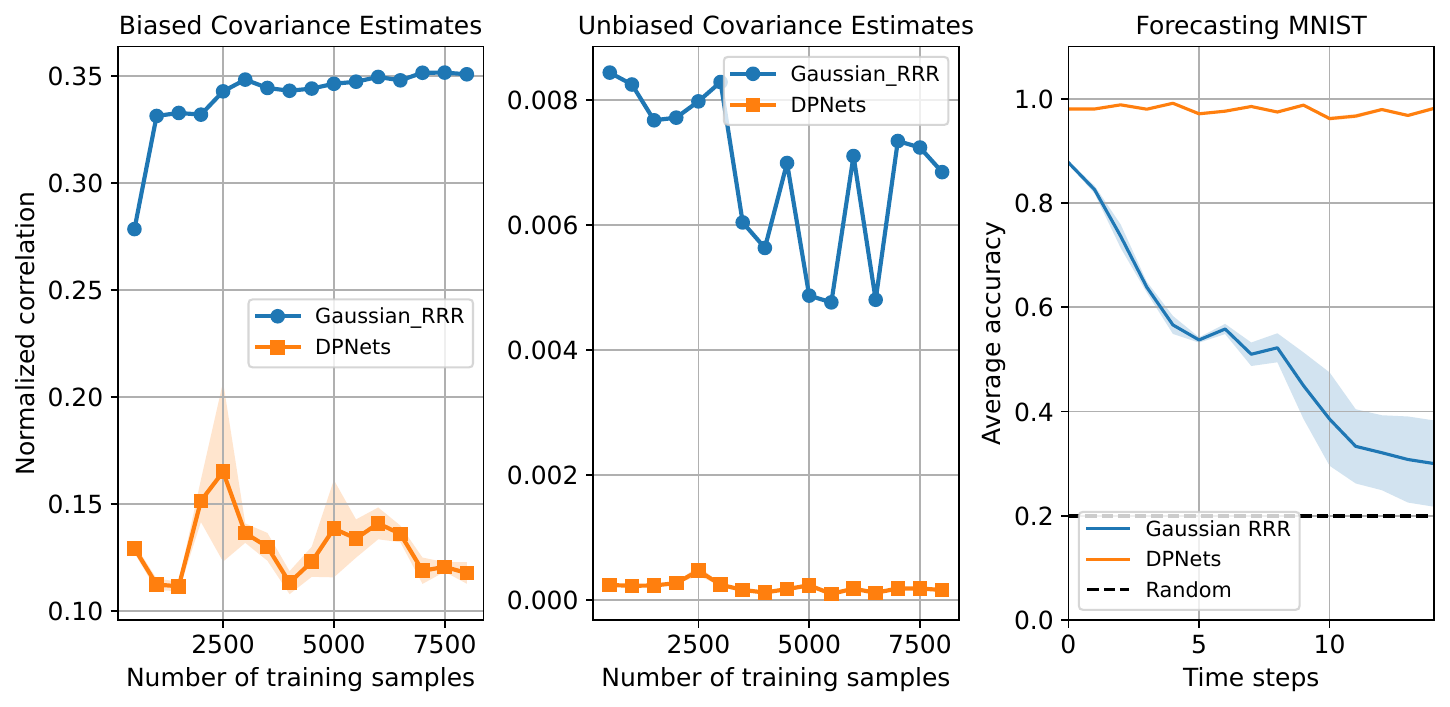}
    \end{subfigure}
    \vspace{0.2cm} 
    \begin{subfigure}[b]{0.9\textwidth} 
        \centering
        \includegraphics[width=\linewidth]{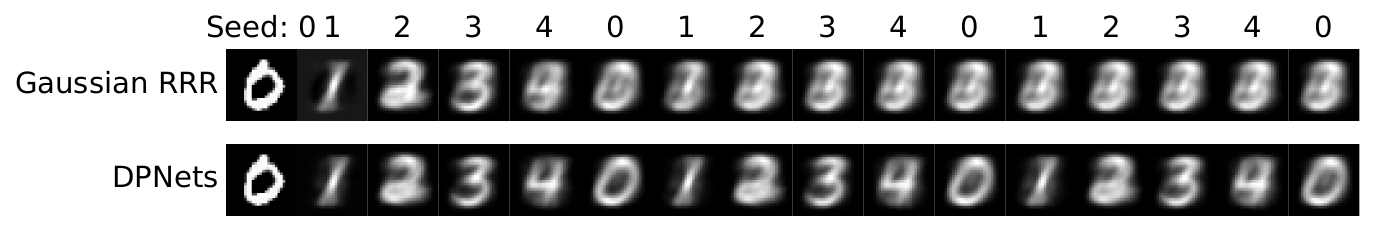}
    \end{subfigure}
    \caption{Performance evaluation of rank-5 RRR estimators using Gaussian and DPNet kernels on MNIST with $\eta = 0.05$: normalized correlations and forecast accuracy. }
\end{figure}

\begin{figure}[t]
    \centering
\includegraphics[width=\textwidth]{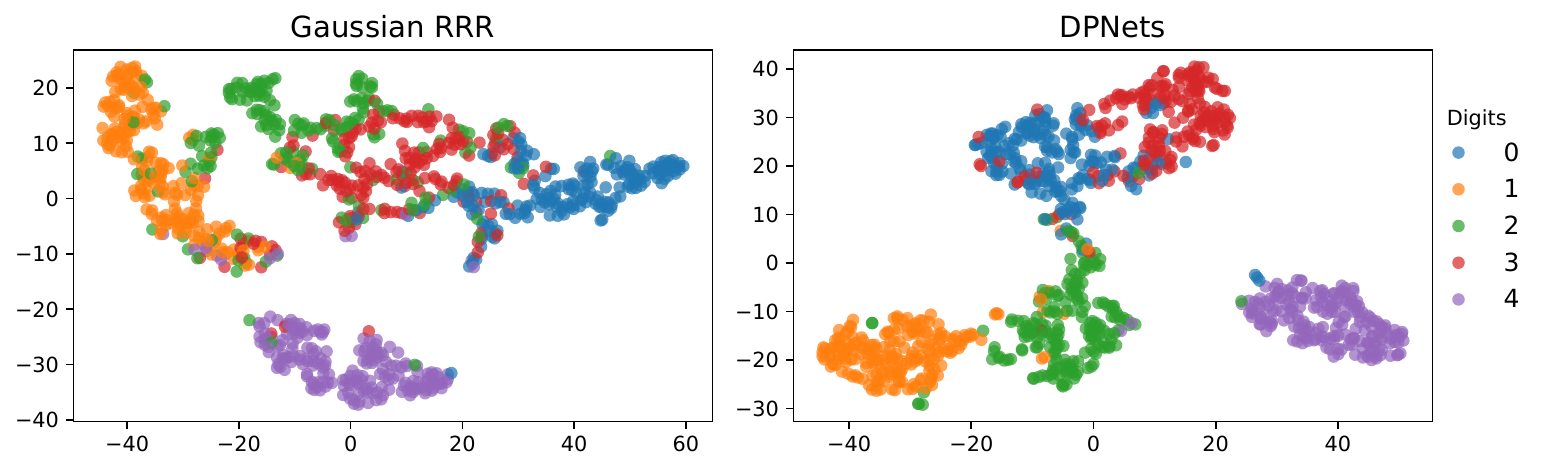}
    \caption{"t-SNE visualization of the concatenated left and right eigenfunctions on the MNIST test dataset with $\eta = 0.05$}
    \label{fig:t-SNE_0.05}
\end{figure}
\end{document}